\documentclass{article}

\PassOptionsToPackage{numbers, compress}{natbib}
\usepackage[final]{nips_2017}

\usepackage[utf8]{inputenc} 
\usepackage[T1]{fontenc}    
\usepackage{hyperref}       
\usepackage{url}            
\usepackage{booktabs}       
\usepackage{amsfonts}       
\usepackage{nicefrac}       
\usepackage{microtype}      
\usepackage{subcaption}
\usepackage{natbib}
\usepackage{setspace}

\usepackage{amsmath}
\usepackage{algorithm}
\usepackage[noend]{algorithmic}

\usepackage[utf8]{inputenc} 
\usepackage[T1]{fontenc}    
\usepackage{hyperref}       
\usepackage{url}            
\usepackage{booktabs}       
\usepackage{amsfonts}       
\usepackage{nicefrac}       
\usepackage{microtype}      

\title{Online Influence Maximization under Independent Cascade Model with Semi-Bandit Feedback}

\author{
\hspace{0.7cm}  Zheng Wen\\ 
\hspace{0.7cm} Adobe Research\\
\hspace{0.7cm} \texttt{zwen@adobe.com} \\
\And
\hspace{1.7cm}Branislav Kveton\\
\hspace{1.7cm}Adobe Research\\
\hspace{1.7cm}\texttt{kveton@adobe.com}
\AND
  Michal Valko \\
SequeL team, INRIA Lille - Nord Europe\\
\texttt{michal.valko@inria.fr}
 \And
  Sharan Vaswani \\
University of British Columbia \\
\texttt{sharanv@cs.ubc.ca}
}

\newcommand{\bbP}{\mathbb{P}}

\newcommand{\imb}{\tt{IMLinUCB}}

\newcommand{\oracle}{{\tt{ORACLE}}}
\newcommand{\cucb}{{\tt{CUCB}}}

\newcommand{\rnd}[1]{\mathsf{#1}}

\newtheorem{lemma}{Lemma}
\newtheorem{definition}{Definition}
\newtheorem{theorem}{Theorem}
\newtheorem{remark}{Remark}

\newenvironment{proof}{\paragraph{Proof:}}{\hfill$\square$}
\newcommand{\metric}{maximum observed relevance\xspace} 

\newcommand{\CommaBin}{\mathbin{\raisebox{0.5ex}{,}}}

\newcommand{\set}[1]{\left\{#1\right\}}

\newcommand{\E}{\mathbb{E}}

\newcommand{\abs}[1]{\left|#1\right|}

\newcommand{\transpose}{^\mathsf{\scriptscriptstyle T}}

\newcommand{\cE}{\mathcal{E}}

\newcommand{\cG}{\mathcal{G}}
\newcommand{\cH}{\mathcal{H}}

\newcommand{\cO}{\mathcal{O}}

\newcommand{\cS}{\mathcal{S}}

\newcommand{\cV}{\mathcal{V}}

\newcommand{\bI}{{\bf I}}
\newcommand{\bM}{{\bf M}}

\newcommand{\bV}{{\bf V}}
\newcommand{\bw}{{\bf w}}

\newcommand{\bX}{{\bf X}}

\renewcommand{\epsilon}{\varepsilon}

\renewcommand{\tilde}{\widetilde}
\renewcommand{\bar}{\overline}

\newcommand{\nothere}[1]{}

\usepackage{xspace}

\newcommand{\UCB}{\texttt{UCB}\xspace}

\usepackage[colorinlistoftodos, textwidth=30mm, shadow, textsize=small]{todonotes}
\definecolor{babyblue}{rgb}{0.54, 0.81, 0.94}
\definecolor{citrine}{rgb}{0.89, 0.82, 0.04}
\definecolor{misocolor}{rgb}{0.16,0.27,0.86}

\begin{document}

\maketitle

\begin{abstract}
We study the online influence maximization problem in social networks under the \emph{independent cascade} model. Specifically, we aim to learn the set of ``best influencers'' in a social network online while repeatedly interacting with it. We address the challenges of (i) combinatorial action space, since the number of feasible influencer sets grows exponentially with the maximum number of influencers, and (ii) limited feedback, since only the influenced portion of the network is observed. Under a stochastic semi-bandit feedback, we propose and analyze $\imb$, a computationally efficient {\tt UCB}-based algorithm. Our bounds on the cumulative regret are polynomial in all quantities of interest, achieve near-optimal dependence on the number of interactions and reflect the \emph{topology} of the network and the \emph{activation probabilities} of its edges, thereby giving insights on the problem complexity. To the best of our knowledge, these are the first such results. Our experiments show that in several representative graph topologies, the regret of $\imb$ scales as suggested by our upper bounds. $\imb$ permits linear generalization and thus is both statistically and computationally suitable for large-scale problems. Our experiments also show that $\imb$ with linear generalization can lead to low regret in real-world online influence maximization.
\end{abstract}

\setstretch{0.99}

\section{Introduction}
\label{sec:introduction}
\newcommand{\etal}{\emph{et al.}}
Social networks are increasingly important as media for spreading information, ideas, and influence. Computational advertising studies models of information propagation or diffusion in such networks~\citep{kempe2003maximizing, chen2010scalable,easley2010networks}. \emph{Viral marketing} aims to use this information propagation to spread awareness about a specific product. More precisely, agents (marketers) aim to select a fixed number of influencers (called \emph{seeds} or \emph{source nodes}) and provide them with free products or discounts. They expect that these users will influence their neighbours and, transitively, other users in the social network to adopt the product. This will thus result in information propagating across the network as more users adopt or become aware of the product. The marketer has a budget on the number of free products and must choose seeds in order to maximize the \emph{influence spread}, which is the expected number of users that become aware of the product. This problem is referred to as \emph{influence maximization} (IM)~\cite{kempe2003maximizing}. 

For IM, the social network is modeled as a directed graph with the nodes representing users, and the edges representing relations (e.g., friendships on Facebook, following on Twitter) between them. Each directed edge $(i, j)$ is associated with an \emph{activation probability} $\bar{w}(i, j)$ that models the strength of influence that user $i$ has on user $j$. We say a node $j$ is a \emph{downstream neighbor} of node $i$ if there is a directed edge $(i,j)$ from $i$ to $j $. The IM problem has been studied under a number of diffusion models~\cite{kempe2003maximizing,gomez2012influence,li2013influence}. The best known and studied are the models in~\cite{kempe2003maximizing}, and in particular the \emph{independent cascade} (IC) model. In this work, we assume that the diffusion follows the IC model and describe it next. 

After the agent chooses a set of source nodes $\cS$, the independent cascade model defines a diffusion (influence) process: At the beginning, all nodes in $\cS$ are activated (influenced); subsequently, every activated node~$i$ can activate its downstream neighbor $j$ with probability $\bar{w}(i, j)$ once, \emph{independently} of the history of the process. This process runs until no activations are possible. In the IM problem, the goal of the agent is to \emph{maximize the expected number of the influenced nodes} subject to a cardinality constraint on $\cS$. Finding the best set $\cS$ is an  NP-hard problem, but under common diffusion models including IC, it can be efficiently approximated to within a factor of $1 - 1 / e$  \cite{kempe2003maximizing}.

In many social networks, however, the activation probabilities are \emph{unknown}. One possibility is to learn these from past propagation data~\cite{saito2008prediction,goyal2010learning,netrapalli2012learning}. However in practice, such data are hard to obtain and the large number of parameters makes this learning challenging. This motivates the learning framework of IM bandits \cite{vaswani2015influence,valko2016bandits,vaswani2017model}, where the agent needs to learn to choose a good set of source nodes \emph{while} 
repeatedly interacting with the network. Depending on the feedback to the agent, the IM bandits can have (1) full-bandit feedback, where only the \emph{number of influenced nodes} is observed; (2) node semi-bandit feedback, where the \emph{identity of influenced nodes} is observed; or (3) edge semi-bandit feedback, where the \emph{identity of influenced edges} (edges going out from influenced nodes) is observed. In this paper, we give results for the edge semi-bandit feedback model, where we observe for each influenced node, the downstream neighbors that this node influences. Such feedback is feasible to obtain in most online social networks. These networks track activities of users, for instance, when a user retweets a tweet of another user. They can thus trace the propagation (of the tweet) through the network, thereby obtaining edge semi-bandit feedback. 

The IM bandits problem combines two main challenges. First, the number of actions (possible sets)~$\cS$ grows \emph{exponentially} with the cardinality constraint on~$\cS$. Second, the agent can only observe the influenced portion of the network as feedback. Although IM bandits have been studied in the past~\cite{lei2015online,chen2015combinatorial,vaswani2015influence,carpentier2016revealing,vaswani2017model} (see Section~\ref{sec:related work} for an overview and comparison), there are a number of open challenges~\cite{valko2016bandits}. One challenge is to identify reasonable \emph{complexity metrics} that depend on both the topology and activation probabilities of the network and characterize the information-theoretic complexity of the IM bandits problem. Another challenge is to develop learning algorithms such that (i) their performance scales gracefully with these metrics and (ii) are computationally efficient and can be applied to large social networks with millions of users. 

In this paper, we address these two challenges under the IC model with access to edge semi-bandit feedback. We refer to our model as an \emph{independent cascade semi-bandit (ICSB)}. We make four main contributions. First, we propose $\imb$, a {\tt UCB}-like algorithm for ICSBs that permits linear generalization and is suitable for large-scale problems. Second, we define a new complexity metric, referred to as \emph{maximum observed relevance} for ICSB, which depends on the topology of the network and is a non-decreasing function of activation probabilities. The maximum observed relevance $C_*$ can also be upper bounded based on the network topology or the size of the network in the worst case. However, in real-world social networks, due to the relatively low activation probabilities~\cite{goyal2010learning},~$C_*$ attains much smaller values as compared to the worst case upper bounds. Third, we bound the cumulative regret of $\imb$. Our regret bounds are polynomial in all quantities of interest and have near-optimal dependence on the number of interactions. They reflect the structure and activation probabilities of the network through $C_*$ and do not depend on inherently large quantities, such as the reciprocal of the minimum probability of being influenced (unlike~\cite{chen2015combinatorial}) and the cardinality of the action set. Finally, we evaluate $\imb$ on several problems. Our empirical results on simple representative topologies show that the regret of $\imb$ scales as suggested by our topology-dependent regret bounds. We also show that $\imb$ with linear generalization can lead to low regret in real-world online influence maximization problems. 

\newcommand{\Sopt}{\cS^{\mathrm{opt}}}
\section{Influence Maximization under Independence Cascade Model}
\label{sec:imp}

In this section, we define notation and give the formal problem statement for the IM problem under the IC model.
Consider a directed graph $\cG=\left( \cV, \cE \right)$
with a set $\cV=\left \{ 1, 2, \ldots, L\right\}$ of $L=|\cV|$ nodes,
a set $\cE= \left \{ 1, 2, \ldots, |\cE|\right \}$ of directed edges, and an arbitrary \emph{binary}
weight function $\bw:\cE \rightarrow \{ 0,1 \}$. 
We say that a node $v_2 \in \cV$ is \emph{reachable} from a node 
$v_1 \in \cV$ under $\bw$ if there is a directed path\footnote{As is standard in graph theory, a directed path is a sequence of directed edges connecting a sequence of distinct nodes, under the restriction that all edges are directed in the same direction.} $p=(e_1, e_2, \ldots, e_l)$ from $v_1$ to $v_2$ in $\cG$ satisfying $\bw(e_i)=1$
for all $i =1,2,\ldots, l$, where $e_i$ is the $i$-th edge in $p$.  
For a given source node set $\cS \subseteq \cV$ and $\bw$, we say that node $v \in \cV$ is \emph{influenced} if $v$ is reachable from at least one source node in $\cS$ under $\bw$; and denote the number of influenced nodes in $\cG$ by $f(\cS, \bw)$. By definition, the nodes in $\cS$ are always influenced.

The influence maximization (IM) problem is characterized by a triple $\left(\cG, K, \bar{w} \right)$,
where $\cG$ is a given directed graph, $K \leq L$ is the cardinality of source nodes, and
$\bar{w}: \cE \rightarrow [0,1]$ is a probability weight function 
mapping each edge $e \in \cE$ to a real number $\bar{w}(e) \in [0,1]$.
The agent needs to choose a set of $K$ source nodes $\cS \subseteq \cV$ based on  $\left(\cG, K, \bar{w} \right)$.
Then a random binary weight function $\bw$, which encodes the diffusion process under the IC model, is obtained by independently sampling a Bernoulli random variable $\bw(e) \sim \mathrm{Bern}\left(\bar{w}(e) \right)$ 
for each edge $e \in \cE$. The agent's objective is to maximize the expected number of the influenced nodes:  $\max_{\cS:\, |\cS|=K} f(\cS, \bar{w})$,
where $f(\cS, \bar{w}) \stackrel{\Delta}{=}\E_{\bw} \left[f(\cS, \bw) \right]$
is the expected number of influenced nodes when the source node set is $\cS$ and
$\bw$ is sampled according to $\bar{w}$.\footnote{Notice that the definitions of $f(\cS, \bar{w})$ and
$f(\cS, \bw)$ are consistent in the sense that if $\bar{w} \in \{0,1\}^{|\cE|}$, then $f(\cS, \bar{w})=f(\cS, \bw)$ with probability $1$.
} 

It is well-known that the (offline) IM problem is NP-hard \cite{kempe2003maximizing}, but can be approximately solved by approximation/randomized algorithms \cite{chen2010scalable} under the IC model. In this paper, we refer to such algorithms as oracles to distinguish them from the machine learning algorithms discussed in following sections.
Let $\Sopt$ be the optimal solution of this problem, and $\cS^*=\oracle (\cG, K, \bar{w})$ be the (possibly random) solution of an oracle $\oracle$. For any $\alpha, \gamma \in [0,1]$, we say that $\oracle$ is an $(\alpha, \gamma)$-approximation oracle for a given $(\cG, K)$ if for any $\bar{w}$, $f(\cS^*, \bar{w}) \geq \gamma f(\Sopt, \bar{w})$ with probability at least $\alpha$. 
Notice that this further implies that $\E \left[f(\cS^*, \bar{w}) \right] \geq \alpha \gamma f(\Sopt, \bar{w})$. We say an oracle is exact if 
$\alpha =\gamma =1$.

\section{Influence Maximization Semi-Bandit}
In this section, we first describe the IM semi-bandit problem. Next, we state the linear generalization assumption and describe $\imb$, our \UCB-based semi-bandit algorithm. 

\subsection{Protocol}
The \emph{independent cascade semi-bandit (ICSB)} problem is also characterized by a triple $\left(\cG, K, \bar{w} \right)$, but $\bar{w}$ is \emph{unknown} to the agent. The agent interacts with the independent cascade semi-bandit for $n$ rounds. At each round $t=1,2,\ldots, n$, the agent first chooses a source node set $\cS_t \subseteq \cV$ with cardinality $K$ based on its prior information and past observations. Influence then diffuses from the nodes in $\cS_t$ according to the IC model. Similarly to the previous section, this can be interpreted as the environment generating a binary weight function $\bw_t$ by independently sampling $\bw_t(e) \sim \mathrm{Bern} \left( \bar{w}(e) \right)$ for each $e \in \cE$. At round $t$, the agent receives the reward $f(\cS_t, \bw_t)$, that is equal to the number of nodes influenced at that round. The agent also receives edge semi-bandit feedback from the diffusion process. Specifically, for any edge $ e =(u_1, u_2) \in \cE$, the agent observes the realization of $\bw_t(e)$ if and only if the start node $u_1$ of the directed edge $e$ is influenced in the realization $\bw_t$. The agent's objective is to maximize the expected cumulative reward over the $n$~steps. 

\subsection{Linear generalization}
Since the number of edges in real-world social networks tends to be in millions or even billions, we need to exploit some generalization model across 
activation probabilities to develop efficient and deployable learning algorithms. In particular, we assume that there exists a linear-generalization model for the probability weight function $\bar{w}$. That is, each edge $e \in \cE $ is associated with a \emph{known} feature vector $x_e \in \Re^d$ (here $d$ is the dimension of the feature vector) and that there is an \emph{unknown} coefficient vector $\theta^* \in \Re^d$ such that for all $e \in \cE$, $\bar{w}(e)$ is ``well approximated" by $x_{e}\transpose \theta^*$. Formally, we assume that $\rho \stackrel{\Delta}{=}\max_{e \in \cE}  | \bar{w}(e) -x_{e}\transpose \theta^*  |$ is small. In Section~\ref{sec:FB-experiment}, we see that such a linear generalization leads to efficient learning in real-world networks. Note that all vectors in this paper are column vectors.

Similar to the existing approaches for linear bandits~\cite{abbasi2011improved, dani2008stochastic}, we exploit the linear generalization to develop a learning algorithm for ICSB. Without loss of generality, we assume that $\| x_{e} \|_2 \leq 1$ for all $e \in \cE$. Moreover, we use $\bX \in \Re^{|\cE| \times d}$ to denote the feature matrix, i.e., the row of $\bX$ associated with edge $e$ is $x_{e}\transpose$. Note that if a learning agent does not know how to construct good features, it can always choose the na\"{i}ve feature matrix $\bX=\bI \in \Re^{|\cE| \times |\cE|}$ and have no generalization model across edges. We refer to the special case $\bX=\bI \in \Re^{|\cE| \times |\cE|}$ as the \emph{tabular} case.

\subsection{$\imb$ algorithm}
In this section, we propose Influence Maximization Linear UCB ($\imb$), detailed in Algorithm~\ref{alg:imb}.
Notice that $\imb$ represents its past observations as a positive-definite matrix (\emph{Gram matrix}) $\bM_t \in \Re^{d \times d}$ and a vector $B_t \in \Re^d$.  Specifically, let $\bX_t$ be a matrix whose rows are the feature vectors of all observed edges in $t$ steps and $Y_t$ be a binary column vector encoding the realizations of all observed edges in $t$ steps. Then $\bM_t=\bI+\sigma^{-2} \bX_t\transpose \bX_t$ and $B_t=\bX_t\transpose Y_t$.

At each round $t$, $\imb$ operates in three steps: First, it computes an upper confidence bound $U_t(e)$ for each edge $e \in \cE$.
Note that $\mathrm{Proj}_{[0,1]}(\cdot)$ projects a real number into interval $[0,1]$ to ensure that $U_t \in [0, 1]^{|\mathcal{E}| }$.
Second, it chooses a set of source nodes based on the given $\oracle$ and $U_t$, which is also a probability-weight function. Finally, it receives the edge semi-bandit feedback and uses it to update $\bM_t$ and $B_t$. 
It is worth emphasizing that $\imb$ is computationally efficient as long as $\oracle$ is computationally efficient. 
Specifically, at each round $t$, the computational complexities of both Step 1 and 3 of $\imb$ are $\cO \left(|\cE| d^2 \right)$.\footnote{Notice that in a practical implementation, we store $\bM_t^{-1}$ instead of $\bM_t$. Moreover, $\bM_t \leftarrow \bM_t + \sigma^{-2} x_{e} x_{e}\transpose $ is equivalent to $\bM_t^{-1} \leftarrow \bM_t^{-1} -\frac{\bM_t^{-1} x_e x_e\transpose \bM_t^{-1}}{ x_e\transpose \bM_t^{-1} x_e + \sigma^2}$.}

\begin{algorithm}[t]
\caption{$\imb$: Influence Maximization Linear \UCB}
\label{alg:imb}
\begin{algorithmic}
\STATE \textbf{Input:} graph $\cG$, source node set cardinality $K$, oracle $\oracle$, feature vector $x_e$'s, and algorithm parameters $\sigma, c>0$,
\vspace{0.15cm}
\STATE \textbf{Initialization:} $B_0 \leftarrow 0 \in \Re^d$, $\bM_0 \leftarrow \mathbf{I} \in \Re^{d \times d}$
\vspace{0.15cm}
\FOR{$t=1,2,\ldots,n$}
\STATE 1. set $\bar{\theta}_{t-1} \leftarrow \sigma^{-2} \bM_{t-1}^{-1} B_{t-1}$ and the UCBs as
$
U_t(e) \leftarrow \mathrm{Proj}_{[0,1]} \left( x_{e}\transpose \bar{\theta}_{t-1} + c \sqrt{x_{e}\transpose \bM_{t-1}^{-1} x_{e}} \right)
$
for all $ e \in \cE$
\STATE 2. choose $\cS_t \in \oracle(\cG, K, U_t)$, and observe the edge-level semi-bandit feedback
\STATE 3. update statistics:
\STATE \quad (a) initialize $\bM_t \leftarrow \bM_{t-1}$ and $B_t \leftarrow B_{t-1}$
\STATE \quad (b) for all observed edges $e \in \cE$, update $\bM_t \leftarrow \bM_t + \sigma^{-2} x_{e} x_{e}\transpose$ and $B_t \leftarrow B_t + x_{e} \bw_t(e)$
\ENDFOR
\end{algorithmic}
\end{algorithm}

It is worth pointing out that in the tabular case, $\imb$ reduces to $\cucb$ \cite{chen13combinatorial}, 
in the sense that the confidence radii in $\imb$ are the same as those in $\cucb$, up to logarithmic factors.
That is, $\cucb$ can be viewed as a special case of $\imb$ with $\bX=\bI$.

\subsection{Performance metrics}
Recall that the agent's objective is to maximize the expected cumulative reward, which is equivalent to minimizing the expected cumulative regret. The cumulative regret is the loss in reward (accumulated over rounds) because of the lack of knowledge of the activation probabilities. Observe that in each round $t$, $\imb$ needs to use an approximation/randomized algorithm $\oracle$ for solving the offline IM problem. Naturally, this can lead to $\cO(n)$ cumulative regret, since at each round there is a non-diminishing regret due to the approximation/randomized nature of $\oracle$. To analyze the performance of $\imb$ in such cases, we define a more appropriate performance metric, the scaled cumulative regret, as
$R^{\eta}(n)=\sum_{t=1}^n \E \left[ \textstyle R_t^{\eta} \right]$, where $n$ is the number of steps, $\eta>0$ is the scale, and $R_t^{\eta} =  f(\Sopt, \bw_t) -\frac{1}{\eta} f(\cS_t, \bw_t)$ is the $\eta$-scaled realized regret $R_t^{\eta}$ at round~$t$. When $\eta=1$, $R^{\eta}(n)$ reduces to the standard expected cumulative regret $R(n)$.

\newcommand{\comp}{C_*}
\newcommand{\mre}{E_*}

\section{Analysis}
\label{sec:analysis}

In this section, we give a regret bound for $\imb$ for the case when $\bar{w}(e) = x_{e}\transpose \theta^*$ for all $e \in \cE$, i.e., the linear generalization is perfect. 
Our main contribution is a 
 regret bound that scales with a 
new complexity metric, \emph{\metric}, which depends on \emph{both} the topology of $\cG$ and the probability weight function $\bar{w}$, and is defined in Section~\ref{sec:complexity}. We highlight 
this as most known results for this problem are worst case, 
and some of them do not depend on probability weight function at all.

\newcommand{\compG}{C_{\cG}}

\subsection{Maximum observed relevance}
\label{sec:complexity}

We start by defining some terminology. For given directed graph $\cG = (\cV, \cE)$ and source node set $\cS \subseteq \cV$,
we say an edge $e \in \cE$ is \emph{relevant} to a node $v \in \cV \setminus \cS$ under $\cS$
 if there exists a path $p$ from a source node
$s \in \cS$ to $v$ such that (1) $e \in p$ and (2) $p$ does not contain another source node other than $s$.
Notice that with a given $\cS$, whether or not a node $v \in \cV \setminus \cS$ is influenced only depends on the binary weights~$\bw$ on its relevant edges.
For any edge $e \in \cE$, we define $N_{\cS, e}$ as the number of nodes in $\cV \setminus \cS$ 
it is relevant to, and define $P_{\cS, e}$ as the conditional probability that $e$ is observed given $\cS$,
\begin{align}
\label{eqn:N}
N_{\cS, e}  \stackrel{\Delta}{=} \textstyle \sum_{v \in \cV \setminus \cS} \mathbf{1} \left \{ \text{$e$ is relevant to $v$ under $\cS$} \right \}
\quad \text{and} \quad P_{\cS, e}  \stackrel{\Delta}{=} \bbP \left( \text{$e$ is observed}\ \middle | \, \cS \right) .
\end{align}
Notice that $N_{\cS,e}$ only depends on the topology of $\cG$, while $P_{\cS, e}$ depends on \emph{both} the topology of~$\cG$ and the probability weight $\bar{w}$.
The \emph{\metric} $\comp$ is defined as the maximum (over~$\cS$) 2-norm of $N_{\cS, e}$'s weighted by $P_{\cS, e}$'s,
\begin{align}
\label{eqn:complexity}
\comp \stackrel{\Delta}{=} \textstyle \max_{\cS:\, |\cS|=K}\sqrt{ \sum_{e \in \cE} N^2_{\cS, e} P_{\cS, e} }.
\end{align}
As is detailed in the proof of Lemma~\ref{lemma:worst} in Appendix~\ref{sec:proof_main}, $\comp$ arises in the step where Cauchy-Schwarz inequality is applied. 
Note that $\comp$ also depends on both the topology of $\cG$ and the probability weight~$\bar{w}$.
However, 
$\comp$ can be bounded from above only based on the topology of $\cG$ or the size of the problem, i.e., $L=|\cV|$ and 
$|\cE|$. Specifically, by defining $\compG \stackrel{\Delta}{=} \textstyle \max_{\cS:\, \abs{\cS} = K} \sqrt{\sum_{e \in \cE} N_{\cS, e}^2} $, we have
\begin{equation}
\comp \leq
\compG = \textstyle
\max_{\cS:\, \abs{\cS} = K} \sqrt{\sum_{e \in \cE} N_{\cS, e}^2} \leq (L - K) \sqrt{|\cE|} =\cO\left( L \sqrt{|\cE|} \right) = \cO\left(L^2 \right),
\end{equation}
where $\compG$ is the maximum/worst-case (over $\bar{w}$) $\comp$ for the directed graph $\cG$, and the maximum is obtained by 
setting $\bar{w}(e)=1$ for all $e \in \cE$.
Since $\compG$ is worst-case, it might be very far away from~$\comp$ if the activation probabilities are small.
Indeed, this is what we expect in typical real-world situations.
Notice also that if $\max_{e \in \cE} \bar{w}(e) \rightarrow 0$, then 
$P_{\cS, e} \rightarrow 0$ for all~$e \notin \cE(\cS)$  and $P_{\cS, e}=1$ for all~$e \in \cE(\cS)$,
where $\cE(\cS)$ is the set of edges with start node in~$\cS$, hence we have
$
\comp \rightarrow C_{\cG}^0 \stackrel{\Delta}{=} \textstyle \max_{\cS:\, |\cS|=K}\sqrt{ \sum_{e \in \cE(\cS)} N^2_{\cS, e}  }$.
In particular, if $K$ is small,
$C_{\cG}^0$ is much less than $\compG$ in many topologies.
For example, in a complete graph with $K=1$, $\compG= \Theta(L^2)$ while $C_{\cG}^0= \Theta(L^{\frac{3}{2}})$.
Finally, it is worth pointing out that there exist situations $(\cG, \bar{w})$ such that $\comp = \Theta(L^2)$. One such example is when $\cG$ is a complete graph
with $L$ nodes and $\bar{w}(e) = L/(L+1)$ for all edges $e$ in this graph.

\begin{figure}
  \centering
  \includegraphics[ scale=1]{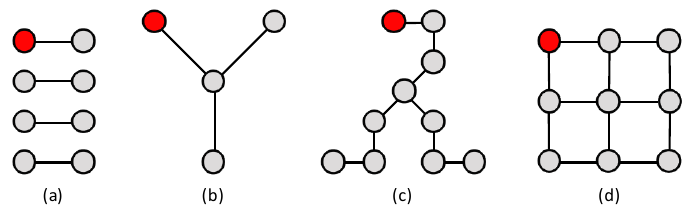} \\
  \caption{
  {\bf a}. Bar graph on $8$ nodes. {\bf b}. Star graph on $4$ nodes. {\bf c}. Ray graph on $10$ nodes.
  {\bf d}. Grid graph on $9$ nodes. Each undirected edge denotes two directed edges in opposite directions.}
  \label{fig:graphs}
  \vspace{-0.05in}
\end{figure}

To give more intuition, in the rest of this subsection, we illustrate how $\compG$, the \emph{worst-case} $\comp$, varies with four graph topologies in Figure~\ref{fig:graphs}:  bar,  star, ray, and  grid, as well as two other topologies: general tree and complete graph.
We fix the node set $\cV=\{1,2, \ldots, L \}$ for all graphs. The bar graph (Figure~\ref{fig:graphs}a) is a graph where nodes $i$ and $i + 1$ are connected when $i$ is odd. The star graph (Figure~\ref{fig:graphs}b) is a graph where node $1$ is central and all remaining nodes $i \in \cV \setminus \set{1}$ are connected to it. The distance between any two of these nodes is $2$. The ray graph (Figure~\ref{fig:graphs}c) is a star graph with $k = \left\lceil\sqrt{L - 1}\right\rceil$ arms, where node $1$ is central and each arm contains either  $\lceil (L - 1)/k \rceil$ or $\lfloor (L - 1)/k \rfloor$ nodes connected in a line. The distance between any two nodes in this graph is $\cO(\sqrt{L})$.
The grid graph (Figure~\ref{fig:graphs}d) is a classical non-tree graph with $\cO(L)$ edges.

To see 
how $\compG$ varies with the graph topology, we start with the simplified case when $K=\abs{\cS} = 1$.
In the bar graph (Figure~\ref{fig:graphs}a), only one edge is relevant to a node $v \in \cV \setminus \cS$ and all the other edges are not relevant to any nodes.
Therefore, $\compG \leq 1$.
In the star graph (Figure~\ref{fig:graphs}b), for any $s$, at most one edge is relevant to at most $L-1$ nodes and the remaining edges are relevant to at most one
node.
In this case, $\compG \leq \sqrt{L^2 + L} = \cO(L)$.
In the ray graph (Figure~\ref{fig:graphs}c), for any $s$, at most $\cO(\sqrt{L})$ edges are relevant to $L-1$ nodes and the remaining edges are relevant to at most 
$\cO(\sqrt{L})$ nodes. In this case, $\compG = \cO(\sqrt{L^\frac{1}{2} L^2 + L L}) = \cO(L^\frac{5}{4})$. 
Finally, recall that for all graphs we can bound $\compG$ by $\cO (L \sqrt{|\cE|})$, regardless of $K$.
Hence, for the grid graph (Figure~\ref{fig:graphs}d) and general tree graph, $\compG=\cO(L^{\frac{3}{2}})$ since $|\cE|=\cO(L)$;
for the complete graph $\compG=\cO(L^{2})$ since $|\cE|=\cO(L^2)$.
Clearly,~$\compG$ varies widely with the topology of the graph. The second column of Table~\ref{table:topology} summarizes how $\compG$ varies with 
the above-mentioned graph topologies for  general $K=\abs{\cS}$.

\subsection{Regret guarantees}
\label{sec:regret bound}

Consider  
$\comp$ defined in Section~\ref{sec:complexity} and 
recall the worst-case upper bound $\comp \leq  (L -K) \sqrt{|\cE|}$,
we have the following regret guarantees for $\imb$.
\begin{theorem}
\label{thm:main}
Assume that (1)  $\bar{w}(e) = x_{e}\transpose \theta^*$ for all $e \in \cE$ and (2) $\oracle$ is an $(\alpha, \gamma)$-approximation algorithm. Let $D$ be a known upper bound on $\|\theta^* \|_2$, if we apply $\imb$ with $\sigma=1$ and
\begin{equation}
\label{eqn:c_lower}
c = \sqrt{d \log \left ( 1+ \frac{n |\cE|}{d}\right) + 2 \log \left ( n(L+1-K) \right)} + D,
\end{equation}
then we have
\begin{align}
R^{\alpha \gamma}(n) \leq  & \, \frac{2c \comp}{\alpha \gamma}  \sqrt{ d n |\cE| \log_2 \left( 1+ \frac{n |\cE|}{d}\right)}
 +  1  
 = \, \tilde{\cO} \left( d \comp \sqrt{|\cE| n} /(\alpha \gamma)  \right)  \label{bound:general} \\
 \leq & \, \tilde{\cO} \left( d (L-K) |\cE| \sqrt{n} /(\alpha \gamma)  \right). \label{bound:size}
\end{align}
Moreover, if the feature matrix $\bX=\bI \in \Re^{|\cE| \times |\cE|}$ (i.e., the tabular case), we have
\begin{align}
R^{\alpha \gamma}(n) \leq &  \,  \frac{2c \comp}{\alpha \gamma}  \sqrt{ n |\cE| \log_2 \left( 1+ n \right)}
 +  1 
 = \, \tilde{\cO}  \left( |\cE| \comp \sqrt{n} /(\alpha \gamma) \right) \label{bound:tabular} \\
 \leq & \, \tilde{\cO} \left( (L-K) |\cE|^{\frac{3}{2}} \sqrt{n} /(\alpha \gamma)  \right).  \label{bound:tabular_size} 
\end{align}
\end{theorem}
Please refer to Appendix~\ref{sec:proof_main} for the proof of Theorem~\ref{thm:main}, that we outline in Section~\ref{sec:proof_sketch}.
We now briefly comment on the regret bounds  in Theorem~\ref{thm:main}.

\noindent \textbf{Topology-dependent bounds:} Since $\comp$ is topology-dependent, the regret bounds in Equations~\ref{bound:general} and~\ref{bound:tabular}
are also topology-dependent. Table~\ref{table:topology} summarizes the regret bounds for each topology\footnote{The regret bound for bar graph is based on Theorem~\ref{thm:stronger} in the appendix, which is a stronger version of Theorem~\ref{thm:main} for disconnected graph.} discussed in Section~\ref{sec:complexity}.
Since the regret bounds in Table~\ref{table:topology} are the worst-case regret bounds for a given topology, more general topologies have larger regret bounds.
For instance, the regret bounds for tree are larger than their counterparts for star and ray, since star and ray are special trees. 
The grid and tree can also be viewed as special complete graphs by setting $\bar{w}(e)=0$ for some $e \in \cE$, hence
complete graph has larger regret bounds. Again, in practice we expect $\comp$ to be far smaller due to activation probabilities.

\begin{table}[t]
\begin{center}
\begin{tabular}{|c|c|c|c|}
\hline
\hline
topology  & $\compG$ (worst-case $\comp$) & $R^{\alpha \gamma}(n)$ for general $\bX$ &  $R^{\alpha \gamma}(n)$ for $\bX=\bI$ \\
\hline
\hline
bar graph & $\cO(\sqrt{K})$ & $\tilde{\cO} \left( d K \sqrt{n} / (\alpha \gamma) \right)$ &  $\tilde{\cO} \left( L \sqrt{K n} / (\alpha \gamma)  \right)$ \\
\hline
star graph &$\cO( L \sqrt{K})$ & $ \tilde{\cO} \left( d L^{\frac{3}{2}} \sqrt{Kn} / (\alpha \gamma)  \right)$  &  $ \tilde{\cO} \left( L^{2} \sqrt{Kn} / (\alpha \gamma)  \right)$  \\
\hline
ray graph & $\cO( L^{\frac{5}{4}} \sqrt{K})$ &  $\tilde{\cO} \left( d L^{\frac{7}{4}} \sqrt{Kn} / (\alpha \gamma)  \right)$  & $ \tilde{\cO} \left( L^{\frac{9}{4}} \sqrt{Kn} / (\alpha \gamma)  \right)$ \\
\hline
tree graph &$\cO(L^{\frac{3}{2}})$ & $\tilde{\cO} \left( d L^{2} \sqrt{n} / (\alpha \gamma)  \right)$  &  $\tilde{\cO} \left(  L^{\frac{5}{2}} \sqrt{n} / (\alpha \gamma)  \right)$ \\
\hline
grid graph & $\cO(L^{\frac{3}{2}})$ & $\tilde{\cO} \left( d L^{2} \sqrt{n} / (\alpha \gamma)  \right)$ & $\tilde{\cO} \left(  L^{\frac{5}{2}} \sqrt{n} / (\alpha \gamma)  \right)$  \\
\hline
complete graph &$\cO(L^{2})$ &$\tilde{\cO} \left( d L^3 \sqrt{n} / (\alpha \gamma)  \right) $ & $\tilde{\cO} \left( L^4 \sqrt{n} / (\alpha \gamma)  \right) $   \\
\hline
\hline
\end{tabular}
\end{center}
\caption{$\compG$ and \emph{worst-case} regret bounds for different graph topologies.}
\label{table:topology}
\end{table}%

\noindent \textbf{Tighter bounds in tabular case and under exact oracle:} Notice that for the tabular case with feature matrix $\bX=\bI$ and $d = |\cE|$, $\tilde{\cO} ( \sqrt{|\cE|} )$ tighter regret bounds are obtained in Equations~\ref{bound:tabular} and~\ref{bound:tabular_size}.
Also notice that the $\tilde{\cO}(1/(\alpha \gamma))$ factor is due to the fact that $\oracle$ is an
$(\alpha, \gamma)$-approximation oracle. If $\oracle$ solves the IM problem exactly (i.e., $\alpha=\gamma=1$), then $R^{\alpha \gamma}(n)=R(n)$.

\textbf{Tightness of our regret bounds:}
First, 
note that our regret bound in 
the bar case with $K=1$ matches the regret bound of the classic ${\tt LinUCB}$ algorithm. Specifically, with perfect linear generalization, this case is equivalent to a linear bandit problem with~$L$ arms and feature dimension $d$. 
From Table~\ref{table:topology}, our regret bound in this case is $\tilde{\cO} \left( d  \sqrt{ n} \right)$, 
which matches the known regret bound of ${\tt LinUCB}$ that can be obtained by the technique of~\cite{abbasi2011improved}.
Second, we briefly discuss the tightness of the regret bound in Equation~\ref{bound:size} for a general graph with
$L$ nodes and $|\cE|$ edges. Note that the $\tilde{\cO}(\sqrt{n})$-dependence on 
time is near-optimal, and the $\tilde{\cO}(d)$-dependence on feature dimension 
is standard in linear bandits \cite{abbasi2011improved,wen15efficient}, since $\tilde{\cO}(\sqrt{d})$ results are only known for impractical algorithms. The $\tilde{\cO}(L-K)$ factor is due to the fact that the reward in this problem
is from $K$ to $L$, rather than from $0$ to $1$. To explain the $\tilde{\cO}(|\cE|)$ factor in this bound, notice that
one $\tilde{\cO}(\sqrt{|\cE|})$ factor is due to the fact that at most $\tilde{\cO}(|\cE|)$ edges might be observed at each round (see Theorem~\ref{theorem:graph}), and is intrinsic to the problem similarly to combinatorial semi-bandits \cite{kveton15tight};
another $\tilde{\cO}(\sqrt{|\cE|})$ factor is due to linear generalization (see Lemma~\ref{lemma:worst}) and might be removed by better analysis.
We conjecture that our $ \tilde{\cO} \left( d (L-K) |\cE| \sqrt{n} /(\alpha \gamma)  \right)$ regret bound in this case is at most
$\tilde{\cO} (\sqrt{|\cE| d} )$ away from being tight.

\subsection{Proof sketch}
\label{sec:proof_sketch}
We now outline the proof of Theorem~\ref{thm:main}.
For each round $t \leq n$, we define the favorable event
$\xi_{t-1}=  \{
|x_{e}\transpose (\bar{\theta}_{\tau-1} - \theta^*)| \leq c \sqrt{x_{e}\transpose \bM_{\tau-1}^{-1} x_{e}}, \, \forall e \in \cE, \, \forall \tau \leq t
 \}$, and the unfavorable event $\bar{\xi}_{t-1}$ as the complement of $\xi_{t-1}$.
If we decompose $\E [R^{\alpha \gamma}_t] $, the $(\alpha \gamma)$-scaled 
expected regret at round $t$, over events $\xi_{t-1}$ and $\bar{\xi}_{t-1}$, 
and bound $R^{\alpha \gamma}_t$ on event $\bar{\xi}_{t-1}$ using the na\"{i}ve bound
$R_t^{\alpha \gamma} \leq L-K$, then,
\begin{align}
\E [R^{\alpha \gamma}_t] \leq  \, \bbP\left ( \xi_{t-1} \right) \E \left[  R^{\alpha \gamma}_t \middle | \xi_{t-1} \right]+ \, \bbP\left ( \bar{\xi}_{t-1} \right) [L-K]. 
\nonumber
\end{align}
By choosing $c$ as specified by Equation~\ref{eqn:c_lower}, we have $\bbP\left ( \bar{\xi}_{t-1} \right) [L-K] < 1/n$ (see Lemma~\ref{lemma:concentration} in the appendix).
On the other hand, notice that by definition of $\xi_{t-1}$, $\bar{w}(e) \leq U_t (e)$, $\forall e \in \cE$ under event $\xi_{t-1}$.
Using the monotonicity of~$f$ in the probability weight, and the fact  that $\oracle$ is an $(\alpha, \gamma)$-approximation algorithm, we have
\begin{align}
 \E \left[  R^{\alpha \gamma}_t \middle | \xi_{t-1} \right] \leq  \E \left[  f(\cS_t,  U_t)- f(\cS_t, \bar{w}) \middle | \xi_{t-1} \right] /(\alpha \gamma).
 \nonumber
\end{align}

The next observation is that, from the linearity of expectation, the gap $f(\cS_t,  U_t)- f(\cS_t, \bar{w})$ decomposes over nodes
$v \in \cV \setminus \cS_t$. Specifically, for any source node set $\cS \subseteq \cV$, any probability weight function $w: \cE \rightarrow [0,1]$, and any node $v \in \cV$, we define $f(\cS, w, v)$ as the probability that
node~$v$ is influenced if the source node set is $\cS$ and the probability weight is $w$. 
Hence, we have
\[
f(\cS_t,  U_t)- f(\cS_t, \bar{w}) =  \textstyle \sum_{v \in \cV \setminus \cS_t} \left[  f(\cS_t,  U_t, v)- f(\cS_t, \bar{w}, v)\right]
.\]  
In the appendix, we show that under any weight function, the diffusion process from the source node set $S_t$ to the target node $v$ can be modeled as a Markov chain. Hence, weight function $U_t$ and $\bar{w}$ give us two Markov chains with the same state space but different transition probabilities.  
$f(S_t, U_t, v) - f(S_t, \bar{w}, v)$ can be recursively bounded based on the state diagram of the Markov chain under weight function $\bar{w}$. With some algebra, 
Theorem~\ref{theorem:graph} in Appendix~\ref{sec:proof_main} bounds $f(\cS_t,  U_t, v) - f(\cS_t, \bar{w}, v)$ 
by the edge-level gap $ U_t (e) -\bar{w} (e)$
on the observed relevant edges for node $v$,
\begin{align}
\label{eqn:outline_1}
 f(\cS_t,  U_t, v) - f(\cS_t, \bar{w}, v) 
\leq  \textstyle \sum_{e \in \cE_{\cS_t, v}} \E \left[
\mathbf{1} \left \{ O_{t}(e) \right \} \left[ U_t (e) -\bar{w} (e) \right]  \middle | \cH_{t-1}, \cS_t
\right],
\end{align}
for any $t$, any ``history" (past observations) $\cH_{t-1}$ and $\cS_t$ such that $\xi_{t-1}$ holds, and any $v \in \cV \setminus \cS_t$,
where $\cE_{\cS_t, v}$ is the set of edges relevant to $v$
and 
$O_{t}(e)$ is the event that edge $e$ is observed at round $t$. 
Based on Equation~\ref{eqn:outline_1}, we can prove Theorem~\ref{thm:main} using the standard linear-bandit techniques
(see Appendix~\ref{sec:proof_main}).

\section{Experiments}
In this section, we present a synthetic experiment in order to empirically validate our upper bounds on the regret. Next, we evaluate our algorithm on a real-world Facebook subgraph. 

\subsection{Stars and rays}

\begin{figure}
    \centering
    \begin{subfigure}[b]{0.6\textwidth}
        \includegraphics[width=\textwidth]{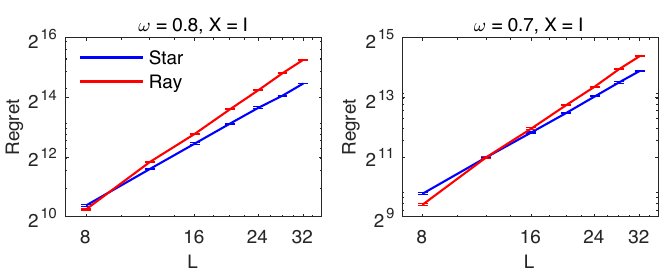}
        \caption{Stars and rays: The log-log plots of the $n$-step regret of $\imb$ in two graph topologies after $n = 10^4$ steps. We vary the number of nodes $L$ and the mean edge weight $\omega$.}
        \label{fig:growth}
    \end{subfigure}
    ~ 
    \begin{subfigure}[b]{0.35\textwidth}
        \includegraphics[width=\textwidth]{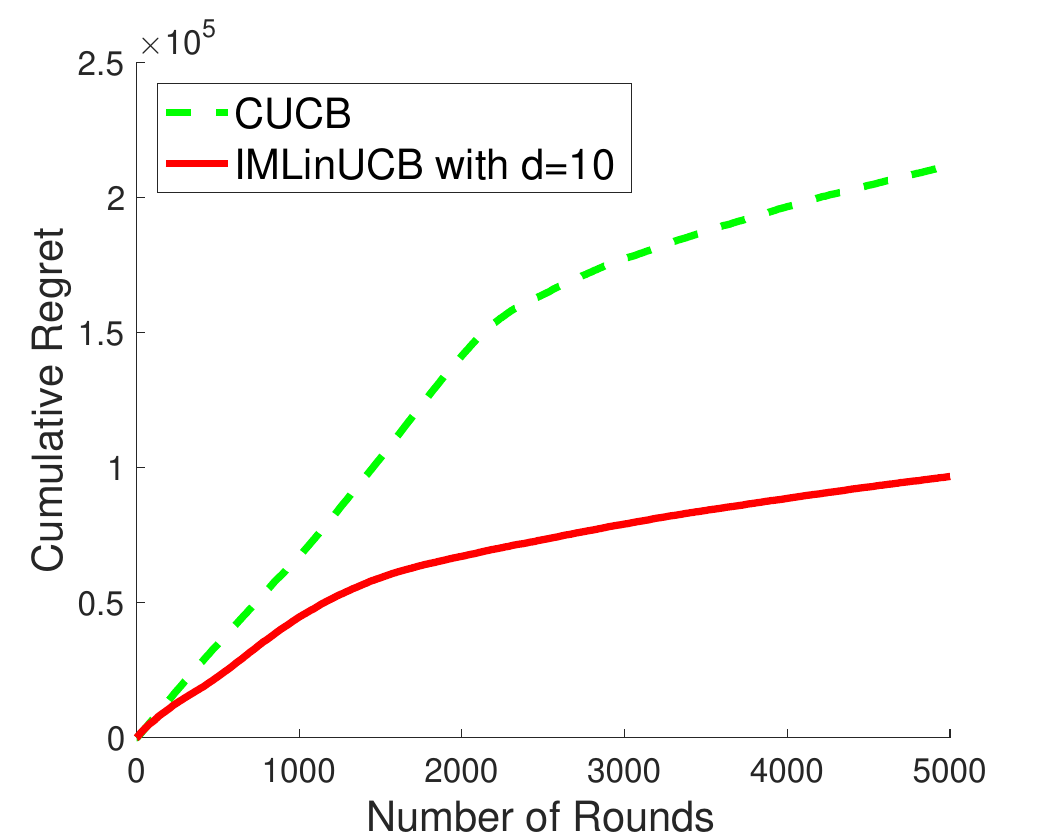}
        \caption{Subgraph of Facebook network}
        \label{fig:FB-regret}
    \end{subfigure}
    \caption{Experimental results}
    \label{fig:experiments}
    \vspace{-0.05in}
\end{figure}
In the first experiment, we evaluate $\imb$ on undirected stars and rays
(Figure~\ref{fig:graphs}) and validate that the regret grows with the number of nodes $L$ and the maximum observed relevance $\comp$ as shown in Table~\ref{table:topology}. We focus on the tabular case ($\bX =\bI$) with $K=\abs{\cS} = 1$, where 
the IM problem can be solved exactly. We vary the number of nodes $L$;  and edge weight $\bar{w}(e)=\omega$, which is the same for all edges~$e$.  We run $\imb$ for $n = 10^4$ steps and verify that it converges to the optimal solution in each experiment. We report the $n$-step regret of $\imb$ for $8 \leq L \leq 32$ in Figure~\ref{fig:growth}. Recall that from Table~\ref{table:topology}, 
$R(n) = \tilde{\cO}(L^2)$ for star and $R(n) = \tilde{\cO}(L^\frac{9}{4})$ for ray.

We numerically  estimate the growth of regret in $L$, the exponent of $L$, in the log-log space of $L$ and regret. In particular, since $\log(f(L)) = p \log(L) + \log(c)$ for any $f(L) = c L^p$ and $c > 0$, both $p$ and $\log(c)$ can be estimated by linear regression in the new space. For star graphs with $\omega = 0.8$ and $\omega = 0.7$, our estimated growth are respectively $\cO(L^{2.040})$ and $\cO(L^{2.056})$, which are close to the expected $\tilde{\cO}(L^2)$. 
For ray graphs with $\omega = 0.8$ and $\omega = 0.7$, our estimated growth are respectively $\cO(L^{2.488})$ and $\cO(L^{2.467})$, which are again close to the expected $\tilde{\cO}(L^\frac{9}{4})$. 
This shows that maximum observed relevance $\comp$ proposed in Section~\ref{sec:complexity} is a reasonable complexity metric for these two topologies.

\subsection{Subgraph of Facebook network}
\label{sec:FB-experiment}
In the second experiment,
we demonstrate the potential performance gain of $\imb$ in real-world influence maximization semi-bandit problems by exploiting linear generalization across edges. Specifically, we compare $\imb$ with $\cucb$ in a subgraph of Facebook network from~\cite{snapnets}.
The subgraph has $L =\vert \cV \vert = 327$ nodes and $\vert \cE \vert = 5038$ directed edges.
Since the true probability weight function $\bar{w}$ is not available, we independently sample $\bar{w}(e)$'s
from the uniform distribution $U(0,0.1)$ and treat them as ground-truth. Note that this range of probabilities is guided by empirical evidence in~\cite{goyal2010learning,barbieri2013topic}. We set $n = 5000$ and $K = 10$ in this experiment.
For $\imb$, we choose $d = 10$ and generate edge feature $x_e$'s as follows: we first use
${\tt node2vec}$ algorithm~\cite{grover2016node2vec} to generate a node feature in $\Re^d$ for each node $v \in \cV$;
then for each edge $e$, we generate $x_e$ as the element-wise product of node features of the two nodes connected to $e$. Note that the linear generalization in this experiment is imperfect in the sense that
$\min_{\theta \in \Re^d} \max_{e \in \cE} |\bar{w}(e) - x_e^T \theta| >0$.
For both $\cucb$ and $\imb$, we choose $\oracle$ as the state-of-the-art offline IM algorithm
proposed in~\cite{Tang2014Influence}. 
To compute the cumulative regret, we compare against a fixed seed set $\cS^*$ obtained by using the
true $\bar{w}$ as input to the oracle proposed in~\cite{Tang2014Influence}. We average the empirical cumulative regret over $10$ independent runs,
and plot the results in Figure~\ref{fig:FB-regret}. The experimental results show that compared with $\cucb$, $\imb$ can significantly reduce the cumulative regret by exploiting linear generalization across $\bar{w}(e)$'s.

\section{Related Work}
\label{sec:related work}
There exist prior results on IM semi-bandits~\cite{lei2015online,chen2015combinatorial,vaswani2015influence}. First, Lei \etal~\cite{lei2015online} gave algorithms for the same feedback model as ours. The algorithms are not analyzed and cannot solve large-scale problems because they estimate each edge weight independently. Second, our setting is a special case of stochastic combinatorial semi-bandit with a submodular reward function and stochastically observed edges~\cite{chen2015combinatorial}.
 Their work is the closest related work.  Their gap-dependent and gap-free bounds are both problematic because they depend on the reciprocal of the minimum observation probability~$p^*$ of an edge: Consider a line graph with $|\cE|$ edges where all edge weights are $0.5$. Then $1/p^*$ is $2^{|\cE| - 1}$. On the other hand, our derived regret bounds in Theorem~\ref{thm:main} are
 polynomial in all quantities of interest. 
 A very recent result of Wang and Chen \cite{DBLP:journals/corr/WangC17a} removes the $1/p^*$ factor in~\cite{chen2015combinatorial}  for the tabular case and presents a worst-case bound of $\tilde{\cO}(L |\cE| \sqrt{n})$, which in the 
 tabular complete graph case improves over our result by $\tilde{\cO}(L)$. 
 On the other hand, their analysis does not give structural guarantees that we provide with maximum observed relevance $C_*$ obtaining potentially much better results for the case in hand and  giving insights for the complexity of IM bandits.
Moreover, both Chen \etal~\cite{chen2015combinatorial} and Wang and Chen \cite{DBLP:journals/corr/WangC17a} do not consider generalization models across edges or nodes, and therefore their proposed algorithms are unlikely to be practical for real-world social networks. In contrast, our proposed algorithm scales to large problems by exploiting linear generalization across edges.

\noindent
\textbf{IM bandits for different influence models and settings:} 
There exist a number of extensions and related results for IM bandits. We 
only mention the most related ones (see \cite{valko2016bandits} for a recent survey).
Vaswani \etal~\cite{vaswani2015influence} proposed a learning algorithm for a different and more challenging feedback model, where the learning agent observes influenced \emph{nodes but not the edges}, but they do not give any guarantees. Carpentier and Valko \cite{carpentier2016revealing} give a minimax optimal algorithm for IM bandits but only consider a \emph{local model} of influence with a \emph{single} source  and a cascade of influences never happens. In related networked bandits~\citep{fang2014networked}, the learner chooses a node and its reward is the \emph{sum} of the rewards of the chosen node and its neighborhood.
 The problem gets more challenging when we allow the influence probabilities to change \citep{bao2016online}, when we allow the seed set to be chosen adaptively \citep{vaswani2016adaptive}, or when we consider a continuous model  \citep{farajtabar2016multistage}. Furthermore, Sigla \etal\,\cite{singla2015information} treats the IM setting with an additional observability constraints, 
where we face a restriction on which nodes we can choose at each round.  This setting is also related to the \emph{volatile multi-armed bandits} where the set of possible arms changes \citep{bnaya2013social}.
Vaswani \etal~\cite{vaswani2017model} proposed a diffusion-independent algorithm for 
IM semi-bandits with a wide range of
diffusion models, based on the maximum-reachability approximation. Despite its wide applicability, 
the maximum reachability approximation introduces an additional approximation factor to the scaled regret bounds. As they have discussed, this approximation factor can be large in some cases.
Lagr\'ee \etal\,\cite{lagree2017algorithms} treat a  \emph{persistent} extension of IM bandits when some 
nodes become persistent over the rounds and no longer yield rewards. 
This work is also a generalization and extension of recent work on cascading bandits \cite{kveton15cascading, kveton15combinatorial, ZongNSKWK16}, since cascading bandits can be viewed as variants of online influence maximization problems with special topologies (chains).

{\small
\paragraph{\small Acknowledgements}
\label{sec:Acknowledgements}
The research presented was supported by French Ministry of
Higher Education and Research, Nord-Pas-de-Calais Regional Council,
Inria and Univert\"at Potsdam associated-team north-european project Allocate, and French National Research Agency projects ExTra-Learn (n.ANR-14-CE24-0010-01) and BoB (n.ANR-16-CE23-0003). We would also like to thank Dr.\,Wei Chen and Mr.\,Qinshi Wang for pointing out a mistake in an earlier version of this paper.}

\newpage

\setstretch{1}

\bibliographystyle{plain}
\bibliography{library-DONOTEDIT,Brano}

\begin{thebibliography}{10}

\bibitem{abbasi2011improved}
Yasin Abbasi-Yadkori, D{\'{a}}vid P{\'{a}}l, and Csaba Szepesv{\'{a}}ri.
\newblock {Improved algorithms for linear stochastic bandits}.
\newblock In {\em Neural Information Processing Systems}, 2011.

\bibitem{bao2016online}
Yixin Bao, Xiaoke Wang, Zhi Wang, Chuan Wu, and Francis C.~M. Lau.
\newblock {Online influence maximization in non-stationary social networks}.
\newblock In {\em International Symposium on Quality of Service}, apr 2016.

\bibitem{barbieri2013topic}
Nicola Barbieri, Francesco Bonchi, and Giuseppe Manco.
\newblock Topic-aware social influence propagation models.
\newblock {\em Knowledge and information systems}, 37(3):555--584, 2013.

\bibitem{bnaya2013social}
Zahy Bnaya, Rami Puzis, Roni Stern, and Ariel Felner.
\newblock {Social network search as a volatile multi-armed bandit problem}.
\newblock {\em Human Journal}, 2(2):84--98, 2013.

\bibitem{carpentier2016revealing}
Alexandra Carpentier and Michal Valko.
\newblock {Revealing graph bandits for maximizing local influence}.
\newblock In {\em International Conference on Artificial Intelligence and
  Statistics}, 2016.

\bibitem{chen2010scalable}
Wei Chen, Chi Wang, and Yajun Wang.
\newblock {Scalable influence maximization for prevalent viral marketing in
  large-scale social networks}.
\newblock In {\em Knowledge Discovery and Data Mining}, 2010.

\bibitem{chen13combinatorial}
Wei Chen, Yajun Wang, and Yang Yuan.
\newblock Combinatorial multi-armed bandit: General framework, results and
  applications.
\newblock In {\em International Conference on Machine Learning}, 2013.

\bibitem{chen2015combinatorial}
Wei Chen, Yajun Wang, and Yang Yuan.
\newblock {Combinatorial multi-armed bandit and its extension to
  probabilistically triggered arms}.
\newblock {\em Journal of Machine Learning Research}, 17, 2016.

\bibitem{dani2008stochastic}
Varsha Dani, Thomas~P Hayes, and Sham~M Kakade.
\newblock {Stochastic linear optimization under bandit feedback}.
\newblock In {\em Conference on Learning Theory}, 2008.

\bibitem{easley2010networks}
David Easley and Jon Kleinberg.
\newblock {Networks, Crowds, and Markets: Reasoning About a Highly Connected
  World}.
\newblock Cambridge University Press, 2010.

\bibitem{fang2014networked}
Meng Fang and Dacheng Tao.
\newblock {Networked bandits with disjoint linear payoffs}.
\newblock In {\em International Conference on Knowledge Discovery and Data
  Mining}, 2014.

\bibitem{farajtabar2016multistage}
Mehrdad Farajtabar, Xiaojing Ye, Sahar Harati, Le~Song, and Hongyuan Zha.
\newblock {Multistage campaigning in social networks}.
\newblock In {\em Neural Information Processing Systems}, 2016.

\bibitem{gomez2012influence}
M~Gomez~Rodriguez, B~Sch{\"o}lkopf, Langford~J Pineau, et~al.
\newblock Influence maximization in continuous time diffusion networks.
\newblock In {\em International Conference on Machine Learning}, 2012.

\bibitem{goyal2010learning}
Amit Goyal, Francesco Bonchi, and Laks~VS Lakshmanan.
\newblock Learning influence probabilities in social networks.
\newblock In {\em Proceedings of the third ACM international conference on Web
  search and data mining}, pages 241--250. ACM, 2010.

\bibitem{grover2016node2vec}
Aditya Grover and Jure Leskovec.
\newblock node2vec: Scalable feature learning for networks.
\newblock In {\em Knowledge Discovery and Data Mining}. ACM, 2016.

\bibitem{kempe2003maximizing}
David Kempe, Jon Kleinberg, and {\'{E}}va Tardos.
\newblock {Maximizing the spread of influence through a social network}.
\newblock {\em Knowledge Discovery and Data Mining}, page 137, 2003.

\bibitem{kveton15cascading}
Branislav Kveton, Csaba Szepesvari, Zheng Wen, and Azin Ashkan.
\newblock Cascading bandits: Learning to rank in the cascade model.
\newblock In {\em Proceedings of the 32nd International Conference on Machine
  Learning}, 2015.

\bibitem{kveton15combinatorial}
Branislav Kveton, Zheng Wen, Azin Ashkan, and Csaba Szepesvari.
\newblock Combinatorial cascading bandits.
\newblock In {\em Advances in Neural Information Processing Systems 28}, pages
  1450--1458, 2015.

\bibitem{kveton15tight}
Branislav Kveton, Zheng Wen, Azin Ashkan, and Csaba Szepesvari.
\newblock Tight regret bounds for stochastic combinatorial semi-bandits.
\newblock In {\em Proceedings of the 18th International Conference on
  Artificial Intelligence and Statistics}, 2015.

\bibitem{lagree2017algorithms}
Paul Lagr{\'{e}}e, Olivier Capp{\'{e}}, Bogdan Cautis, and Silviu Maniu.
\newblock {Effective large-scale online influence maximization}.
\newblock In {\em International Conference on Data Mining}, 2017.

\bibitem{lei2015online}
Siyu Lei, Silviu Maniu, Luyi Mo, Reynold Cheng, and Pierre Senellart.
\newblock {Online influence maximization}.
\newblock In {\em Knowledge Discovery and Data mining}, 2015.

\bibitem{snapnets}
Jure Leskovec and Andrej Krevl.
\newblock Snap datasets: Stanford large network dataset collection.
\newblock http://snap.stanford.edu/data, jun 2014.

\bibitem{li2013influence}
Yanhua Li, Wei Chen, Yajun Wang, and Zhi-Li Zhang.
\newblock Influence diffusion dynamics and influence maximization in social
  networks with friend and foe relationships.
\newblock In {\em ACM international conference on Web search and data mining}.
  ACM, 2013.

\bibitem{netrapalli2012learning}
Praneeth Netrapalli and Sujay Sanghavi.
\newblock Learning the graph of epidemic cascades.
\newblock In {\em ACM SIGMETRICS Performance Evaluation Review}, volume~40,
  pages 211--222. ACM, 2012.

\bibitem{saito2008prediction}
Kazumi Saito, Ryohei Nakano, and Masahiro Kimura.
\newblock Prediction of information diffusion probabilities for independent
  cascade model.
\newblock In {\em Knowledge-Based Intelligent Information and Engineering
  Systems}, pages 67--75, 2008.

\bibitem{singla2015information}
Adish Singla, Eric Horvitz, Pushmeet Kohli, Ryen White, and Andreas Krause.
\newblock {Information gathering in networks via active exploration}.
\newblock In {\em International Joint Conferences on Artificial Intelligence},
  2015.

\bibitem{Tang2014Influence}
Youze Tang, Xiaokui Xiao, and Shi Yanchen.
\newblock Influence maximization: Near-optimal time complexity meets practical
  efficiency.
\newblock 2014.

\bibitem{valko2016bandits}
Michal Valko.
\newblock {\em {Bandits on graphs and structures}}.
\newblock habilitation, {\'{E}}cole normale sup{\'{e}}rieure de Cachan, 2016.

\bibitem{vaswani2017model}
Sharan Vaswani, Branislav Kveton, Zheng Wen, Mohammad Ghavamzadeh, Laks~VS
  Lakshmanan, and Mark Schmidt.
\newblock Model-independent online learning for influence maximization.
\newblock In {\em International Conference on Machine Learning}, 2017.

\bibitem{vaswani2016adaptive}
Sharan Vaswani and Laks V.~S. Lakshmanan.
\newblock {Adaptive influence maximization in social networks: Why commit when
  you can adapt?}
\newblock Technical report, 2016.

\bibitem{vaswani2015influence}
Sharan Vaswani, Laks. V.~S. Lakshmanan, and {Mark Schmidt}.
\newblock {Influence maximization with bandits}.
\newblock In {\em NIPS workshop on Networks in the Social and Information
  Sciences 2015}, 2015.

\bibitem{DBLP:journals/corr/WangC17a}
Qinshi Wang and Wei Chen.
\newblock {Improving regret bounds for combinatorial semi-bandits with
  probabilistically triggered arms and its applications}.
\newblock In {\em Neural Information Processing Systems}, mar 2017.

\bibitem{wen15efficient}
Zheng Wen, Branislav Kveton, and Azin Ashkan.
\newblock Efficient learning in large-scale combinatorial semi-bandits.
\newblock In {\em International Conference on Machine Learning}, 2015.

\bibitem{ZongNSKWK16}
Shi Zong, Hao Ni, Kenny Sung, Nan~Rosemary Ke, Zheng Wen, and Branislav Kveton.
\newblock Cascading bandits for large-scale recommendation problems.
\newblock In {\em Uncertainty in Artificial Intelligence}, 2016.

\end{thebibliography}

\newpage 
\onecolumn
\newcommand{\tr}{\mathrm{trace}}
\newenvironment{oneshot}[1]{\@begintheorem{#1}{\unskip}}{\@endtheorem}

\appendix

\begin{center}
\Large{\bf Appendix}
\end{center}

\section{Proof of Theorem~\ref{thm:main}}
\label{sec:proof_main}

In the appendix, we prove a slightly stronger version of Theorem~\ref{thm:main}, which also uses another complexity metric $\mre$ defined as follows: Assume that the graph  $\cG=(\cV, \cE)$ includes $m$ disconnected subgraphs $\cG_1=(\cV_1, \cE_1), \cG_2=(\cV_2, \cE_2), \ldots, \cG_{m}=(\cV_{m}, \cE_{m})$, 
which are in the descending order based on the number of nodes $|\cE_i|$'s.
We define $\mre$ as the number of edges in the first $\min \{m, K \}$ subgraphs:
\begin{align}
\label{eqn:mre_def}
\mre = \sum_{i=1}^{\min \{m, K \}} |\cE_i|.
\end{align}
Note that by definition, $\mre \leq |\cE|$. Based on $\mre$, we have the following slightly stronger version of Theorem~\ref{thm:main}. 
\begin{theorem}
\label{thm:stronger}
Assume that (1)  $\bar{w}(e) = x_{e}\transpose \theta^*$ for all $e \in \cE$ and (2) $\oracle$ is an $(\alpha, \gamma)$-approximation algorithm. Let $D$ be a known upper bound on $\|\theta^* \|_2$. If we apply $\imb$ with $\sigma=1$ and
\begin{equation}
\label{eqn:stronger:c_lower}
c \geq \sqrt{d \log \left ( 1+ \frac{n \mre}{d}\right) + 2 \log \left ( n(L+1-K) \right)} + D,
\end{equation}
then we have
\begin{align}
\label{bound:stronger:general}
R^{\alpha \gamma}(n) \leq  \, \frac{2c \comp}{\alpha \gamma}  \sqrt{ d n \mre \log_2 \left( 1+ \frac{n \mre}{d}\right)}
 +  1  
 = \, \tilde{\cO} \left( d \comp \sqrt{\mre n}/(\alpha \gamma) \right).
\end{align}
Moreover, if the feature matrix is of the form $X=I \in \Re^{|\cE| \times |\cE|}$ (i.e., the tabular case), we have
\begin{align}
\label{bound:stronger:tabular}
R^{\alpha \gamma}(n) \leq  \,  \frac{2c \comp}{\alpha \gamma}  \sqrt{ n |\cE| \log_2 \left( 1+ n \right)}
 +  1 
 = \, \tilde{\cO}  \left( |\cE| \comp \sqrt{n} /(\alpha \gamma)  \right).
\end{align}
\end{theorem}
Since $\mre \leq |\cE|$, Theorem~\ref{thm:stronger} implies Theorem~\ref{thm:main}.
We prove Theorem~\ref{thm:stronger} in the remainder of this section.

We now define some notation to simplify the exposition throughout this section.
\begin{definition}
\label{def:f_v}
For any source node set $\cS \subseteq \cV$, any probability weight function $w: \cE \rightarrow [0,1]$, and any node $v \in \cV$, we define $f(\cS, w, v)$ as the probability that node $v$ is influenced if the source node set is $\cS$ and the probability weight function is $w$. 
\end{definition}
Notice that by definition, $f(\cS, w) = \sum_{v \in \cV} f(\cS, w, v)$ always holds.
Moreover,  if $v \in \cS$, then $f(\cS, w, v)=1$ for any $w$ by the definition of the influence model.

\begin{definition}
\label{def:O_t}
For any round $t$ and any directed edge  $e \in \cE$, we define event 
\[
O_{t}(e)=\{\text{edge $e$ is observed at round $t$}\}.
\]
\end{definition}
Note that by definition, an directed edge $e$ is observed if and only if its start node is influenced
and observed does not necessarily mean that the edge is \emph{active}. 

\subsection{Proof of Theorem~\ref{thm:stronger}}
\begin{proof}
Let $\cH_t$ be the history ($\sigma$-algebra) of past observations and actions by the end of round $t$.
By the definition of $R_t^{\alpha \gamma} $, we have
\begin{align}
\E\left[ R_t^{\alpha \gamma} \middle | \cH_{t-1} \right]=& f(\Sopt ,  \bar{w})- \frac{1}{\alpha \gamma} \E \left[ f( \cS_t, \bar{w}) 
\middle | \, \cH_{t-1} \right],
\end{align}
where the expectation is over the possible randomness of $\cS_t$, since $\oracle$ might be a randomized algorithm. Notice that the randomness coming from the edge activation is already taken care of in the 
definition of $f$.
For any $t \leq n$, we define event $\xi_{t-1}$ as
\begin{align}
\label{eqn:F_t}
\xi_{t-1}= \left \{
|x_{e}\transpose (\bar{\theta}_{\tau-1} - \theta^*)| \leq c \sqrt{x_{e}\transpose \bM_{\tau-1}^{-1} x_{e}}, \, \forall e \in \cE, \, \forall \tau \leq t
\right \},
\end{align}
and $\bar{\xi}_{t-1}$ as the complement of $\xi_{t-1}$. Notice that $\xi_{t-1}$ is $\cH_{t-1}$-measurable.
Hence we have
\[
\E [R^{\alpha \gamma}_t] \leq \bbP\left ( \xi_{t-1} \right) \E \left[  f(\Sopt, \bar{w})- f(\cS_t, \bar{w})/(\alpha \gamma) \middle | \xi_{t-1} \right] + \bbP\left ( \bar{\xi}_{t-1} \right) [L-K].
\]
Notice that under event $\xi_{t-1}$, $\bar{w}(e) \leq U_t (e)$, $\forall e \in \cE$, for all $t \leq n$, thus we have
\[
f(\Sopt, \bar{w}) \leq f(\Sopt, U_t ) \leq  \max_{\cS: \, |\cS|=K} f(\cS, U_t ) \leq \frac{1}{\alpha \gamma} \E \left[ f( \cS_t, U_t ) 
\middle | \, \cH_{t-1} \right], 
\]
where the first inequality follows from the monotonicity of $f$ in the probability weight, 
and the last inequality follows from the fact that $\oracle$ is an $(\alpha, \gamma)$-approximation algorithm.
Thus, we have
\begin{align}
\E [R^{\alpha \gamma}_t] \leq  \frac{\bbP\left ( \xi_{t-1} \right)}{\alpha \gamma} \E \left[  f(\cS_t,  U_t)- f(\cS_t, \bar{w}) \middle | \xi_{t-1} \right] + \bbP\left ( \bar{\xi}_{t-1} \right) [L-K].
\end{align}
Notice that based on Definition~\ref{def:f_v}, we have
\[
 f(\cS_t,  U_t)- f(\cS_t, \bar{w}) = \sum_{v \in \cV \setminus \cS_t} \left[  f(\cS_t,  U_t, v)- f(\cS_t, \bar{w}, v)\right].
\]
Recall that for a given graph $\cG = (\cV, \cE)$ and a given source node set $\cS \subseteq \cV$, 
we say an edge $e \in \cE$ and a node $v \in \cV \setminus \cS$ are \emph{relevant} if there exists a path $p$ from a source node
$s \in \cS$ to $v$ such that (1) $e \in p$ and (2) $p$ does not contain another source node other than $s$.
We use $\cE_{\cS, v} \subseteq \cE$ to denote the set of edges relevant to node $v$ under the source node set 
$\cS$, and use $\cV_{\cS, v} \subseteq \cV$ to denote the set of nodes connected to at least one edge in 
$\cE_{\cS, v} $. Notice that $\cG_{\cS, v} \stackrel{\Delta}{=} \left( \cV_{\cS, v},\cE_{\cS, v} \right)$  is a subgraph of $\cG$, and we refer to it as
the \textbf{relevant subgraph} of node $v$ under the source node set $\cS$.

Based on the notion of relevant subgraph, we have the following theorem, which bounds $f(\cS_t,  U_t, v) - f(\cS_t, \bar{w}, v)$ by edge-level gaps $ U_t (e) -\bar{w} (e)$
on the observed edges in the relevant subgraph  $\cG_{\cS_t, v}$ for node $v$;
\begin{theorem}
\label{theorem:graph}
For any $t$, any history $\cH_{t-1}$ and $\cS_t$ such that $\xi_{t-1}$ holds, and any $v \in \cV \setminus \cS_t$, we have
\[
f(\cS_t,  U_t, v) - f(\cS_t, \bar{w}, v) \leq \sum_{e \in \cE_{\cS_t, v}} \E \left[
\mathbf{1} \left \{ O_{t}(e) \right \} \left[ U_t (e) -\bar{w} (e) \right]  \middle | \cH_{t-1}, \cS_t
\right],
\]
where $\cE_{\cS_t, v}$ is the edge set of the relevant subgraph $\cG_{\cS_t, v}$.
\end{theorem}

Please refer to Section~\ref{sec:proof_lemma1} for the proof of Theorem~\ref{theorem:graph}.  Notice that  under favorable event
$\xi_{t-1}$, we have $U_t (e) - \bar{w}(e) \leq 2c \sqrt{x_{e}\transpose \bM_{t-1}^{-1} x_{e}}$ for all $e \in \cE$. Therefore, we have
\begin{align}
\E [R^{\alpha \gamma}_t] \leq & \,  \frac{2c}{\alpha \gamma} \bbP\left ( \xi_{t-1} \right) \E \left[  \sum_{v \in \cV \setminus \cS_t} \sum_{e \in \cE_{\cS_t, v}}
\mathbf{1} \{ O_{t}(e) \}  \sqrt{x_{e}\transpose \bM_{t-1}^{-1} x_{e}}  
 \middle | \xi_{t-1} \right] + \bbP\left ( \bar{\xi}_{t-1} \right) [L-K] \nonumber \\
 \leq & \,  \frac{2c}{\alpha \gamma} \E \left[  \sum_{v \in \cV \setminus \cS_t} \sum_{e \in \cE_{\cS_t, v}}
\mathbf{1} \{ O_{t}(e)\}  \sqrt{x_{e}\transpose \bM_{t-1}^{-1} x_{e}}  
  \right] + \bbP\left ( \bar{\xi}_{t-1} \right) [L-K] \nonumber \\
 =& \, \frac{2c}{\alpha \gamma} \E \left[  \sum_{e \in \cE} \mathbf{1} \{ O_{t}(e)\}  \sqrt{x_{e}\transpose \bM_{t-1}^{-1} x_{e}}  
 \sum_{v \in \cV \setminus \cS_t} \mathbf{1} \left \{ e \in \cE_{\cS_t, v} \right \}
  \right] + \bbP\left ( \bar{\xi}_{t-1} \right) [L-K] \nonumber \\
 = & \, \frac{2c}{\alpha \gamma} \E \left[   \sum_{e \in \cE} 
\mathbf{1} \{ O_{t}(e) \}  N_{\cS_t, e} \sqrt{x_{e}\transpose \bM_{t-1}^{-1} x_{e}}  
  \right] + \bbP\left ( \bar{\xi}_{t-1} \right) [L-K],
\end{align}
where $ N_{\cS_t, e}=\sum_{v \in \cV \setminus \cS} \mathbf{1} \left \{ e \in \cE_{\cS_t, v} \right \}$ is defined in Equation~\ref{eqn:N}. Thus we have
\begin{align}
\label{eq:regret_after_ubc}
R^{\alpha \gamma}(n) \leq \frac{2c}{\alpha \gamma} \E \left[  \sum_{t=1}^n \sum_{e \in \cE} 
\mathbf{1} \{ O_{t}(e) \}  N_{\cS_t, e} \sqrt{x_{e}\transpose \bM_{t-1}^{-1} x_{e}}  
  \right] +  [L-K]\sum_{t=1}^n \bbP\left ( \bar{\xi}_{t-1} \right).
\end{align}
In the following lemma, we give a worst-case bound on $\sum_{t=1}^n \sum_{e \in \cE} 
 \mathbf{1} \{ O_{t}(e) \}  N_{\cS_t, e} \sqrt{x_{e}\transpose \bM_{t-1}^{-1} x_{e}}$.
\begin{lemma}
\label{lemma:worst}
For any round $t=1,2,\ldots, n$, we have
\[
\sum_{t=1}^n \sum_{e \in \cE} 
 \mathbf{1} \{ O_{t}(e) \}  N_{\cS_t, e} \sqrt{x_{e}\transpose \bM_{t-1}^{-1} x_{e}} \leq
 \sqrt{\left(\sum_{t=1}^n \sum_{e \in \cE}   \mathbf{1} \{ O_{t}(e) \}  N^2_{\cS_t, e} \right) \frac{d \mre \log \left( 1+ \frac{n \mre}{d \sigma^2}\right)}{\log \left(1 + \frac{1}{\sigma^2} \right)}}\cdot
\]
Moreover, if $X=I \in \Re^{|\cE|  \times |\cE| }$, then we have
\[
\sum_{t=1}^n \sum_{e \in \cE} 
 \mathbf{1} \{ O_{t}(e) \}  N_{\cS_t, e} \sqrt{x_{e}\transpose \bM_{t-1}^{-1} x_{e}}
 \leq
 \sqrt{ \left(\sum_{t=1}^n \sum_{e \in \cE}   \mathbf{1} \{ O_{t}(e) \}  N^2_{\cS_t, e} \right)  \frac{|\cE|  \log \left( 1+ \frac{n}{\sigma^2}\right)}{\log \left(1 + \frac{1}{\sigma^2} \right)}}\cdot
\]
\end{lemma}
Please refer to Section~\ref{sec:proof_lemma2} for the proof of Lemma~\ref{lemma:worst}.
Finally, notice that for any $t$,
\[
\E \left[ \sum_{e \in \cE}   \mathbf{1} \{ O_{t}(e) \}  N^2_{\cS_t, e}  \middle | \cS_t \right]  =
 \sum_{e \in \cE}  N^2_{\cS_t, e}  \E \left[  \mathbf{1} \{ O_{t}(e) \}   \middle | \cS_t \right] =  \sum_{e \in \cE}  N^2_{\cS_t, e} P_{\cS_t, e}  \leq \comp^2,
\]
thus taking the expectation over the possibly randomized oracle and Jensen's inequality, we get
\begin{align}
\E \left[ \sqrt{ \sum_{t=1}^n \sum_{e \in \cE}   \mathbf{1} \{ O_{t}(e) \}  N^2_{\cS_t, e}  }\right] \leq  \,
\sqrt{ \sum_{t=1}^n \E \left[ \sum_{e \in \cE}   \mathbf{1} \{ O_{t}(e) \}  N^2_{\cS_t, e}  \right] } 
\leq  \, \sqrt{ \sum_{t=1}^n \comp^2 }
= \comp \sqrt{n}.
\end{align}
Combining the above with Lemma~\ref{lemma:worst} and \eqref{eq:regret_after_ubc}, we obtain
\begin{align}
R^{\alpha \gamma}(n) \leq \frac{2c \comp}{\alpha \gamma}  \sqrt{ \frac{d n \mre \log \left( 1+ \frac{n \mre}{d \sigma^2}\right)}{\log \left(1 + \frac{1}{\sigma^2} \right)}}
 +  [L-K]\sum_{t=1}^n \bbP\left ( \bar{\xi}_{t-1} \right).
\end{align}
For the special case when $X=I$, we have
\begin{align}
R^{\alpha \gamma}(n) \leq \frac{2c \comp}{\alpha \gamma}  \sqrt{ \frac{ n |\cE| \log \left( 1+ \frac{n}{\sigma^2}\right)}{\log \left(1 + \frac{1}{\sigma^2} \right)}}
 +  [L-K]\sum_{t=1}^n \bbP\left ( \bar{\xi}_{t-1} \right).
\end{align}

Finally, we need to bound the failure probability of upper confidence bound being wrong
$\sum_{t=1}^n \bbP\left ( \bar{\xi}_{t-1} \right)$.
We prove the following bound on $\bbP\left ( \bar{\xi}_{t-1} \right)$:
\begin{lemma}
\label{lemma:concentration}
For any $t=1,2,\ldots, n$, any $\sigma>0$, any $\delta \in (0,1)$, and any
\[
c \geq \frac{1}{\sigma} \sqrt{d \log \left ( 1+ \frac{n \mre}{d \sigma^2}\right) + 2 \log \left (\frac{1}{\delta} \right)} + \|\theta^* \|_2,
\]
we have $\bbP \left( \bar{\xi}_{t-1} \right) \leq \delta$.
\end{lemma}
Please refer to Section~\ref{sec:proof_lemma3} for the proof of Lemma~\ref{lemma:concentration}.
From Lemma~\ref{lemma:concentration}, for a known upper bound~$D$ on $\| \theta^* \|_2$, if we choose $\sigma=1$ and
$c \geq  \sqrt{d \log \left ( 1+ \frac{n \mre}{d}\right) + 2 \log \left ( n(L+1-K) \right)} + D $, which corresponds to $\delta=\frac{1}{n(L+1-K)}$ in Lemma~\ref{lemma:concentration}, then we have
\[
 [L-K]\sum_{t=1}^n \bbP\left ( \bar{\xi}_{t-1} \right) <1.
\]
This concludes the proof of Theorem~\ref{thm:stronger}.
\end{proof}

\subsection{Proof of Theorem~\ref{theorem:graph}}
\label{sec:proof_lemma1}

Recall that we use $\cG_{\cS_t, v} = \left( \cV_{\cS_t, v},\cE_{\cS_t, v} \right)$ to denote the relevant subgraph of node $v$ under the source node set $\cS_t$. Since Theorem~\ref{theorem:graph} focuses on the influence from $\cS_t$ to $v$, and by definition all the paths from $\cS_t$ to $v$ are in $\cG_{\cS_t, v}$, thus, it is sufficient to restrict to $\cG_{\cS_t, v}$ and ignore other parts of $\cG$ in this analysis.

We start by defining some useful notations.\\

\noindent \textbf{Influence Probability with Removed Nodes:}
Recall that for any weight function $w : \cE \rightarrow [0,1]$, any source node set $\cS \subset \cV$ and any target node $v \in \cV$, $f(\cS, w, v)$ is the probability that $\cS$ will influence $v$ under weight $w$ (see Definition~\ref{def:f_v}). We now define a similar notation for the \textbf{influence probability with removed nodes}.
Specifically, for any disjoint node set $\cV_1, \cV_2 \subseteq  \cV_{\cS_t, v} \subseteq \cV$, we define
$h(\cV_1, \cV_2, w)$ as follows:
\begin{itemize}
\item First, we remove nodes $\cV_2$, as well as all edges connected to/from $\cV_2$, from $\cG_{\cS_t, v}$, and obtain a new graph $\cG'$.
\item $h(\cV_1, \cV_2, w)$ is the probability that $\cV_1$ will influence the target node $v$ in graph $\cG'$ under the weight (activation probability) $w(e)$ for all $e \in \cG'$.
\end{itemize}
Obviously, a mathematically equivalent way to define $h(\cV_1, \cV_2, w)$ is to define it
as the probability that $\cV_1$ will influence $v$ in $\cG_{\cS_t, v}$ under a new weight $\tilde{w}$, defined as
\[
\tilde{w}(e)=\left \{
\begin{array}{ll}
0 & \text{if $e$ is from or to a node in $\cV_2$} \\
w(e) & \text{otherwise}
\end{array}
\right.
\]
Note that by definition, $f(\cS_t, w, v)=h(\cS_t, \emptyset, w)$. Also note that $h(\cV_1, \cV_2, w)$ implicitly depends on~$v$, but we omit $v$ in this notation to simplify the exposition.\\

\noindent \textbf{Edge Set $\cE(\cV_1, \cV_2)$:} 
For any two disjoint node sets $\cV_1, \cV_2 \subseteq \cV_{\cS_t, v}$, we define the edge set $\cE(\cV_1, \cV_2)$
as
\[
\cE (\cV_1, \cV_2)=\left \{
e=(u_1, u_2): \, e \in \cE_{\cS_t, v}, \, u_1 \in \cV_1, \text{ and } u_2 \notin \cV_2
\right \}.
\]
That is, $\cE (\cV_1, \cV_2)$ is the set of edges in $\cG_{\cS_t, v}$ from
$\cV_1$ to $\cV_{\cS_t, v} \setminus \cV_2$.\\

\noindent \textbf{Diffusion Process:}
Note that under any edge activation realization $\bw(e)$, $e \in \cE_{\cS_t, v}$, on the relevant subgraph $\cG_{\cS_t, v}$, 
we define a finite-length sequence of disjoint node sets $\cS^0, \cS^1, \ldots, \cS^{\tilde{\tau}}$ as
\begin{align}
\cS^0 \stackrel{\Delta}{=} & \cS_t \nonumber \\
\cS^{\tau+1} \stackrel{\Delta}{=} &  \left \{u_2 \in \cV_{\cS_t, v}: \, u_2 \notin \cup_{\tau'=0}^{\tau} \cS^{\tau'} \text{ and } \exists e=(u_1, u_2) \in \cE_{\cS_t, v}\text{ s.t. } u_1 \in \cS^{\tau} \text{ and }\bw(e)=1 \right \}, 
\end{align}
$\forall \tau=0, \ldots, \tilde{\tau}-1$. That is, under the realization $\bw(e)$, $e \in \cE_{\cS_t, v}$,
$\cS^{\tau+1}$ is the set of nodes directly activated by $\cS^{\tau}$. Specifically, any node $u_2 \in \cS^{\tau+1}$
satisfies $u_2 \notin  \bigcup_{\tau'=0}^{\tau} \cS^{\tau'}$ (i.e. it was not activated before), and there exists an activated edge $e$ from 
$\cS^{\tau}$ to $u_2$ (i.e. it is activated by some node in $\cS^{\tau}$).
We define $\cS^{\tilde{\tau}}$ as the first node set in the sequence s.t. either $\cS^{\tilde{\tau}}=\emptyset$ or 
$v \in \cS^{\tilde{\tau}}$, and assume this sequence terminates at $\cS^{\tilde{\tau}}$. Note that by definition, $\tilde{\tau} \leq |\cV_{\cS_t, v}|$ always holds. We refer to each $\tau=0,1,\ldots, \tilde{\tau}$ as a \textbf{diffusion step} in this section.

To simplify the exposition, we also define $S^{0:\tau} \stackrel{\Delta}{=} \bigcup_{\tau'=0}^{\tau} S^{\tau'}$ for all $\tau \geq 0$ and $S^{0:-1} \stackrel{\Delta}{=} \emptyset$. Since $\bw$ is random, $\left( \cS^{\tau} \right )_{\tau=0}^{\tilde{\tau}}$ is a stochastic process, which we refer to as the \textbf{diffusion process}. Note that $\tilde{\tau}$ is also random; in particular, it is a stopping time. \\

Based on the shorthand notations defined above, we have the following lemma for the diffusion process $\left( \cS^{\tau} \right )_{\tau=0}^{\tilde{\tau}}$ under any weight function $w$:
\begin{lemma}
\label{lemma:transition}
For any weight function $w : \cE \rightarrow [0,1]$, any step $\tau =0, 1, \ldots, \tilde{\tau}$, any $\cS_{\tau}$ and $\cS^{0:\tau-1}$, we have
\[
h \left( \cS^\tau, \cS^{0:\tau-1}, w \right) = \left \{
\begin{array}{ll}
1 & \text{if $v \in \cS^\tau$} \\
0 & \text{if $\cS^\tau = \emptyset$} \\
\E \left[ h \left( \cS^{\tau+1}, \cS^{0:\tau}, w  \right) \middle | ( \cS^\tau, \cS^{0:\tau-1} ) \right] & \text{otherwise}
\end{array}
\right. ,
\]
where the expectation is over $\cS^{\tau+1}$ under weight $w$.
Note that the tuple $( \cS^\tau, \cS^{0:\tau-1} ) $ in the conditional expectation means that $\cS^\tau$ is the source node set
and nodes in $\cS^{0:\tau-1}$ have been removed.
\end{lemma}
\begin{proof}
Notice that by definition, $h \left( \cS^\tau, \cS^{0:\tau-1}, w \right)=1$ if $v \in \cS^{\tau}$ and
$h \left( \cS^\tau, \cS^{0:\tau-1}, w \right)=0$ if $\cS^\tau =\emptyset$. Also note that in these two cases, $\tilde{\tau}=\tau$.

Otherwise, we prove that $h \left( \cS^\tau, \cS^{0:\tau-1}, w \right) = \E \left[ h \left( \cS^{\tau+1}, \cS^{0:\tau}, w  \right) \middle |  (\cS^\tau, \cS^{0:\tau-1})\right]$. 
Recall that by definition, 
 $h \left( \cS^\tau, \cS^{0:\tau-1}, w \right) $ is the probability that $v$ will be influenced conditioning on 
\begin{equation}
\text{source node set $\cS^\tau$ and removed node set $\cS^{0:\tau-1}$,} 
\end{equation}
that is
\begin{equation}
\label{eqn:lemma6:def1}
h \left( \cS^\tau, \cS^{0:\tau-1}, w \right) = \E \left[ \mathbf{1} \left( \text{$v$ is influenced} \right) \middle | (\cS^\tau ,\cS^{0:\tau-1} )\right]
\end{equation}

Let $\bw(e)$, $\forall e \in \cE(\cS^\tau, \cS^{0:\tau})$ be any possible realization.
Now we analyze the probability that $v$ will be influenced conditioning on 
\begin{equation}
\text{source node set $\cS^\tau$, removed node set $\cS^{0:\tau-1}$, and $\bw(e)$
for all $e \in \cE(\cS^\tau, \cS^{0:\tau})$.} \label{eqn:conditions}
\end{equation}
Specifically, conditioning on Equation~\ref{eqn:conditions}, we can define a new weight function $w'$ as
\begin{equation}
w'(e) = \left \{ \begin{array}{ll}
\bw(e) & \text{if $e \in \cE(\cS^\tau, \cS^{0:\tau})$} \\
w(e) & \text{otherwise}
\end{array}
\right.  \label{eqn:w'}
\end{equation}
then $h \left( \cS^\tau, \cS^{0:\tau-1}, w' \right) $ is the probability that $v$ will be influenced conditioning on
Equation~\ref{eqn:conditions}.
That is,
\begin{equation}
\label{eqn:lemma6:def2}
h \left( \cS^\tau, \cS^{0:\tau-1}, w' \right) =  \E \left[ \mathbf{1} \left( \text{$v$ is influenced} \right) \middle | (\cS^\tau ,\cS^{0:\tau-1} ), \bw(e) \, \forall e \in \cE(\cS^\tau, \cS^{0:\tau}) \right],
\end{equation}
for any possible realization of $\bw(e)$, $\forall e \in \cE(\cS^\tau, \cS^{0:\tau})$.
Notice that on the lefthand of Equation~\ref{eqn:lemma6:def2},
$w'$ encodes the conditioning on $\bw(e)$ for all $e \in \cE(\cS^\tau, \cS^{0:\tau})$ (see Equation~\ref{eqn:w'}).

From here to Equation~\ref{eqn:lemma6:eqn4}, we focus on an arbitrary but fixed realization of $\bw(e)$, $\forall e \in \cE(\cS^\tau, \cS^{0:\tau})$ (or equivalently, an arbitrary but fixed
$w'$).
Based on the definition of $\cS^{\tau+1}$, conditioning on Equation~\ref{eqn:conditions},
$\cS^{\tau+1}$ is deterministic and all nodes in $\cS^{\tau+1}$ can also be treated as source nodes.
Thus, we have
\[
h \left( \cS^\tau, \cS^{0:\tau-1}, w' \right) =h \left( \cS^\tau \cup  \cS^{\tau+1} , \cS^{0:\tau-1}, w' \right),
\]
conditioning on Equation~\ref{eqn:conditions}.

On the other hand, conditioning on Equation~\ref{eqn:conditions}, we can treat any edge $e \in \cE(\cS^\tau, \cS^{0:\tau})$ with $\bw(e)=0$ as having been removed.
Since nodes in $\cS^{0:\tau-1}$ have also been removed, and $v \notin \cS^{\tau}$, 
then 
if there is a path from $\cS^{\tau}$ to $v$, then it must go through $\cS^{\tau+1}$, and the last node on the path in $\cS^{\tau+1}$ must be after the last node on the path in $\cS^{\tau}$ (note that the path might come back to $\cS^{\tau}$ for several times).
Hence, conditioning on Equation~\ref{eqn:conditions}, if nodes in $\cS^{\tau+1}$ are also treated as source nodes, then
$\cS^{\tau}$ is irrelevant for influence on $v$ and can be removed. So we have
\begin{equation}
\label{eqn:lemma6:eqn3}
h \left( \cS^\tau, \cS^{0:\tau-1}, w' \right) = h \left( \cS^\tau \cup  \cS^{\tau+1} , \cS^{0:\tau-1}, w' \right) =h \left( \cS^{\tau+1}, \cS^{0:\tau}, w \right) .
\end{equation}
Note that in the last equation we change the weight function back to $w$ since edges in $\cE(\cS^\tau, \cS^{0:\tau})$ have been removed.
Thus, conditioning on Equation~\ref{eqn:conditions}, we have
\begin{align}
\label{eqn:lemma6:eqn4}
h \left( \cS^{\tau+1}, \cS^{0:\tau}, w \right)=& \, h \left( \cS^\tau, \cS^{0:\tau-1}, w' \right) \nonumber \\
=& \, \E \left[ \mathbf{1} \left( \text{$v$ is influenced} \right) \middle | (\cS^\tau ,\cS^{0:\tau-1} ), \bw(e) \, \forall e \in \cE(\cS^\tau, \cS^{0:\tau}) \right].
\end{align}
Notice again that Equation~\ref{eqn:lemma6:eqn4} holds for any possible realization of $\bw(e)$, $\forall e \in \cE(\cS^\tau, \cS^{0:\tau})$.

Finally, we have
\begin{align}
h \left( \cS^\tau, \cS^{0:\tau-1}, w \right) \stackrel{(a)}{=} & \, \E \left[ \mathbf{1} \left( \text{$v$ is influenced} \right) \middle | (\cS^\tau ,\cS^{0:\tau-1} )\right] \nonumber \\
\stackrel{(b)}{=}& \, \E \left[  \E \left[ \mathbf{1} \left( \text{$v$ is influenced} \right) \middle | (\cS^\tau ,\cS^{0:\tau-1} ), \bw(e) \, \forall e \in \cE(\cS^\tau, \cS^{0:\tau}) \right] \middle | (\cS^\tau ,\cS^{0:\tau-1} )\right] \nonumber \\
\stackrel{(c)}{=}& \, \E \left[ h \left( \cS^{\tau+1}, \cS^{0:\tau}, w \right)  \middle | (\cS^\tau ,\cS^{0:\tau-1} )\right],
\end{align}
where (a) follows from Equation~\ref{eqn:lemma6:def1}, (b) follows from the tower rule, and
(c) follows from Equation~\ref{eqn:lemma6:eqn4}. This concludes the proof. 
\end{proof}

Consider two weight functions $U, w: \cE \rightarrow [0,1]$ s.t.\,$U(e) \geq w(e)$ for all $e \in \cE$. The following lemma bounds the difference $h \left( \cS^\tau, \cS^{0:\tau-1}, U \right) - h \left( \cS^\tau, \cS^{0:\tau-1}, w \right)$ in a recursive way.
\begin{lemma}
\label{lemma:recursive_bound}
For any two weight functions $w, U: \cE \rightarrow [0,1]$ s.t. $U(e) \geq w(e)$ for all $e \in \cE$, any step $\tau =0, 1, \ldots, \tilde{\tau}$, any $\cS_{\tau}$ and $\cS^{0:\tau-1}$, we have
\[
h \left( \cS^\tau, \cS^{0:\tau-1}, U \right) - h \left( \cS^\tau, \cS^{0:\tau-1}, w \right) =0 
\]
if $v \in \cS^\tau$ or $\cS^\tau = \emptyset$; and otherwise
\begin{align}
h \left( \cS^\tau, \cS^{0:\tau-1}, U \right) - h \left( \cS^\tau, \cS^{0:\tau-1}, w \right) \leq & \sum_{e \in \cE(\cS^\tau, \cS^{0:\tau})} \left[U(e) - w(e) \right] \nonumber \\
+ & \,
\E \left[ h \left( \cS^{\tau+1}, \cS^{0:\tau}, U  \right) -  h \left( \cS^{\tau+1}, \cS^{0:\tau}, w  \right) \middle | ( \cS^\tau, \cS^{0:\tau-1} ) \right], \nonumber
\end{align}
where the expectation is over $\cS^{\tau+1}$ under weight $w$.
Recall that the tuple $( \cS^\tau, \cS^{0:\tau-1} ) $ in the conditional expectation means that $\cS^\tau$ is the source node set
and nodes in $\cS^{0:\tau-1}$ have been removed.
\end{lemma}
\begin{proof}
First, note that if $v \in \cS^{\tau}$ or $\cS^\tau = \emptyset$, then
\[
h \left( \cS^\tau, \cS^{0:\tau-1}, U \right) - h \left( \cS^\tau, \cS^{0:\tau-1}, w \right) =0 
\]
follows directly from Lemma~\ref{lemma:transition}. Otherwise, 
to simplify the exposition, 
we overload the notation and use $w( \cS^{\tau+1})$ to denote the conditional probability of $\cS^{\tau+1}$ conditioning on $( \cS^\tau, \cS^{0:\tau-1} )$
under the weight function $w$, and similarly for $U( \cS^{\tau+1})$. That is
\begin{align}
w( \cS^{\tau+1}) \stackrel{\Delta}{=}& \,  \mathrm{Prob} \left[  \cS^{\tau+1} \middle | ( \cS^\tau, \cS^{0:\tau-1} ) ; w \right] \nonumber \\
U( \cS^{\tau+1}) \stackrel{\Delta}{=}& \,  \mathrm{Prob} \left[  \cS^{\tau+1} \middle | ( \cS^\tau, \cS^{0:\tau-1} ) ; U \right],
\end{align}
where the tuple $( \cS^\tau, \cS^{0:\tau-1} ) $ in the conditional probability means that $\cS^\tau$ is the source node set
and nodes in $\cS^{0:\tau-1}$ have been removed, and $w$ and $U$ after the semicolon indicate the weight function.

Then from  Lemma~\ref{lemma:transition}, we have
\begin{align}
 h \left( \cS^\tau, \cS^{0:\tau-1}, U \right) =& \, \sum_{\cS^{\tau+1}} U( \cS^{\tau+1}) h \left( \cS^{\tau+1}, \cS^{0:\tau}, U  \right) \nonumber \\
 h \left( \cS^\tau, \cS^{0:\tau-1}, w \right) =& \, \sum_{\cS^{\tau+1}} w( \cS^{\tau+1}) h \left( \cS^{\tau+1}, \cS^{0:\tau}, w  \right)  \nonumber
\end{align}
where the sum is over all possible realization of $\cS^{\tau+1}$. 

Hence we have
\begin{align}
& \, h \left( \cS^\tau, \cS^{0:\tau-1}, U \right) - h \left( \cS^\tau, \cS^{0:\tau-1}, w \right)  \nonumber \\
= & \,  \sum_{\cS^{\tau+1}} \left[
U( \cS^{\tau+1}) h \left( \cS^{\tau+1}, \cS^{0:\tau}, U  \right) - w( \cS^{\tau+1}) h \left( \cS^{\tau+1}, \cS^{0:\tau}, w  \right) 
\right] \nonumber \\
= & \, \sum_{\cS^{\tau+1}} \left[
U( \cS^{\tau+1}) h \left( \cS^{\tau+1}, \cS^{0:\tau}, U  \right) - w( \cS^{\tau+1}) h \left( \cS^{\tau+1}, \cS^{0:\tau}, U  \right) 
\right]  \nonumber \\
+& \, 
\sum_{\cS^{\tau+1}} \left[
w( \cS^{\tau+1}) h \left( \cS^{\tau+1}, \cS^{0:\tau}, U  \right) - w( \cS^{\tau+1}) h \left( \cS^{\tau+1}, \cS^{0:\tau}, w  \right) 
\right]  \nonumber \\
= & \, \sum_{\cS^{\tau+1}} \left[
U( \cS^{\tau+1}) - w( \cS^{\tau+1}) 
\right] h \left( \cS^{\tau+1}, \cS^{0:\tau}, U  \right)   \nonumber \\
+& \, 
\sum_{\cS^{\tau+1}} w( \cS^{\tau+1})  \left[
h \left( \cS^{\tau+1}, \cS^{0:\tau}, U  \right) - h \left( \cS^{\tau+1}, \cS^{0:\tau}, w  \right) 
\right],
\end{align}
where the sum in the above equations is also over all the possible realizations of $\cS^{\tau+1}$.
Notice that by definition, we have
\begin{multline}
\E \left[ h \left( \cS^{\tau+1}, \cS^{0:\tau}, U  \right) -  h \left( \cS^{\tau+1}, \cS^{0:\tau}, w  \right) \middle | ( \cS^\tau, \cS^{0:\tau-1} ) \right] = \\
\sum_{\cS^{\tau+1}} w( \cS^{\tau+1}) \left[
 h \left( \cS^{\tau+1}, \cS^{0:\tau}, U  \right) -  h \left( \cS^{\tau+1}, \cS^{0:\tau}, w  \right) \right] ,
\end{multline}
where the expectation in the lefthand side is over $\cS^{\tau+1}$ under weight $w$, or equivalently, over $\bw(e)$ for all $e \in \cE(\cS^\tau, \cS^{0:\tau})$ under weight $w$.
Thus, to prove Lemma~\ref{lemma:recursive_bound}, it is sufficient to prove that
\begin{equation}
\sum_{\cS^{\tau+1}} \left[
U( \cS^{\tau+1})  - w( \cS^{\tau+1}) 
\right] h \left( \cS^{\tau+1}, \cS^{0:\tau}, U  \right)  \leq \sum_{e \in \cE(\cS^\tau, \cS^{0:\tau})} \left[U(e) - w(e) \right]. 
\end{equation}
Notice that
\begin{align}
& \, \sum_{\cS^{\tau+1}} \left[
U( \cS^{\tau+1})  - w( \cS^{\tau+1}) 
\right] h \left( \cS^{\tau+1}, \cS^{0:\tau}, U  \right) \nonumber \\
\stackrel{(a)}{\leq} & \, \sum_{\cS^{\tau+1}} \left[
U( \cS^{\tau+1})  - w( \cS^{\tau+1}) 
\right] h \left( \cS^{\tau+1}, \cS^{0:\tau}, U  \right) \mathbf{1}\left[ U( \cS^{\tau+1})  \geq w( \cS^{\tau+1}) \right] \nonumber \\
\stackrel{(b)}{\leq} & \, \sum_{\cS^{\tau+1}} \left[
U( \cS^{\tau+1})  - w( \cS^{\tau+1}) 
\right]  \mathbf{1}\left[ U( \cS^{\tau+1})  \geq w( \cS^{\tau+1}) \right] \nonumber \\
\stackrel{(c)}{=} & \, \frac{1}{2} \sum_{\cS^{\tau+1}} \left |
U( \cS^{\tau+1})  - w( \cS^{\tau+1}) 
\right | ,
\end{align}
where (a) holds since
\begin{multline}
\sum_{\cS^{\tau+1}} \left[
U( \cS^{\tau+1})  - w( \cS^{\tau+1}) 
\right] h \left( \cS^{\tau+1}, \cS^{0:\tau}, U  \right) = \\
\sum_{\cS^{\tau+1}} \left[
U( \cS^{\tau+1})  - w( \cS^{\tau+1}) 
\right] h \left( \cS^{\tau+1}, \cS^{0:\tau}, U  \right) \mathbf{1}\left[ U( \cS^{\tau+1})  \geq w( \cS^{\tau+1}) \right]  \\
+
\sum_{\cS^{\tau+1}} \left[
U( \cS^{\tau+1})  - w( \cS^{\tau+1}) 
\right] h \left( \cS^{\tau+1}, \cS^{0:\tau}, U  \right) \mathbf{1}\left[ U( \cS^{\tau+1})  < w( \cS^{\tau+1}) \right], \nonumber
\end{multline}
and the second term on the righthand side is non-positive. And (b) holds since $0 \leq h \left( \cS^{\tau+1}, \cS^{0:\tau}, U  \right) \leq 1 $ by definition.
To prove (c), we define shorthand notations
\begin{align}
A^+=& \sum_{\cS^{\tau+1}} \left[
U( \cS^{\tau+1})  - w( \cS^{\tau+1}) 
\right]  \mathbf{1}\left[ U( \cS^{\tau+1})  \geq w( \cS^{\tau+1}) \right] \nonumber \\
A^-=& \sum_{\cS^{\tau+1}} \left[
U( \cS^{\tau+1})  - w( \cS^{\tau+1}) 
\right]  \mathbf{1}\left[ U( \cS^{\tau+1})  < w( \cS^{\tau+1}) \right] \nonumber
\end{align}
Then we have
\[
A^+ + A^- = \sum_{\cS^{\tau+1}} \left[
U( \cS^{\tau+1})  - w( \cS^{\tau+1}) 
\right] =0,
\]
since by definition $\sum_{\cS^{\tau+1}} U( \cS^{\tau+1}) = \sum_{\cS^{\tau+1}} w( \cS^{\tau+1})=1$. Moreover, we also have
\[
A^+ - A^- = \sum_{\cS^{\tau+1}} \left |
U( \cS^{\tau+1})  - w( \cS^{\tau+1}) 
\right | .
\]
And hence $A^+ = \frac{1}{2} \sum_{\cS^{\tau+1}} \left |
U( \cS^{\tau+1})  - w( \cS^{\tau+1}) 
\right | $. Thus, to prove Lemma~\ref{lemma:recursive_bound}, it is sufficient to prove
\begin{align}
\frac{1}{2} \sum_{\cS^{\tau+1}} \left |
U( \cS^{\tau+1})  - w( \cS^{\tau+1}) 
\right | \leq \sum_{e \in \cE(\cS^\tau, \cS^{0:\tau})} \left[U(e) - w(e) \right].
\end{align}
Let $\tilde{\bw} \in \{0, 1\}^{|\cE(\cS^\tau, \cS^{0:\tau})|}$ be an arbitrary edge activation realization for edges in $\cE(\cS^\tau, \cS^{0:\tau})$.
Also with a little bit abuse of notation, we use $w(\tilde{\bw})$ to denote the probability of $\tilde{\bw}$ under weight $w$. Notice that
\[
w(\tilde{\bw}) = \prod_{e \in \cE(\cS^\tau, \cS^{0:\tau})} w(e)^{\tilde{\bw}(e)} \left[ 1- w(e)\right]^{1 - \tilde{\bw}(e)},
\]
and $U(\tilde{\bw})$ is defined similarly. Recall that by definition $\cS^{\tau+1}$ is a deterministic function of source node set $\cS^\tau$, removed nodes
$\cS^{0:\tau-1}$, and $\tilde{\bw}$. Hence, for any possible realized $\cS^{\tau+1}$, let $\mathbf{W}(\cS^{\tau+1})$ denote the set of $\tilde{\bw}$'s that lead to this $\cS^{\tau+1}$, then we have
\[
U( \cS^{\tau+1}) =\sum_{\tilde{\bw} \in \mathbf{W}(\cS^{\tau+1})} U(\tilde{\bw}) \quad \text{ and } \quad 
w( \cS^{\tau+1}) =\sum_{\tilde{\bw} \in \mathbf{W}(\cS^{\tau+1})} w(\tilde{\bw}) 
\]
Thus, we have
\begin{align}
\frac{1}{2} \sum_{\cS^{\tau+1}} \left |
U( \cS^{\tau+1})  - w( \cS^{\tau+1}) 
\right | = & \, \frac{1}{2} \sum_{\cS^{\tau+1}} \left |
\sum_{\tilde{\bw} \in \mathbf{W}(\cS^{\tau+1})} [ U(\tilde{\bw})
-  w(\tilde{\bw})  ]  \right |  \nonumber \\
\leq & \, 
\frac{1}{2} \sum_{\cS^{\tau+1}} 
\sum_{\tilde{\bw} \in \mathbf{W}(\cS^{\tau+1})} \left | U(\tilde{\bw})
-  w(\tilde{\bw})   \right |  \nonumber \\
= & 
\frac{1}{2}  
\sum_{\tilde{\bw} } \left | U(\tilde{\bw})
-  w(\tilde{\bw})   \right |
\end{align}
Finally, we prove that
\begin{equation}
\frac{1}{2}  
\sum_{\tilde{\bw} } \left | U(\tilde{\bw})
-  w(\tilde{\bw})   \right | \leq 
 \sum_{e \in \cE(\cS^\tau, \cS^{0:\tau})} \left[U(e) - w(e) \right]
\end{equation}
by mathematical induction. Without loss of generality, we order the edges in $\cE(\cS^\tau, \cS^{0:\tau})$ as
$1,2,\ldots, |\cE(\cS^\tau, \cS^{0:\tau})|$. For any $k=1, \ldots,  |\cE(\cS^\tau, \cS^{0:\tau})|$, we use $\tilde{\bw}_k \in \{0, 1\}^{k}$ to denote 
an arbitrary edge activation realization for edges $1,\ldots, k$. Then, we prove 
\begin{equation}
\label{eqn:induction}
\frac{1}{2}  
\sum_{\tilde{\bw}_k} \left | U(\tilde{\bw}_k)
-  w(\tilde{\bw}_k)   \right | \leq 
 \sum_{e =1}^k \left[U(e) - w(e) \right]
\end{equation}
for all $k=1, \ldots,  |\cE(\cS^\tau, \cS^{0:\tau})| $ by mathematical induction. Notice that when $k=1$, we have
\[
\frac{1}{2}  
\sum_{\tilde{\bw}_1} \left | U(\tilde{\bw}_1)
-  w(\tilde{\bw}_1)   \right | =\frac{1}{2} \left[
|U(1) - w(1)| +|(1-U(1)) - (1-w(1))| \right]
= U(1)-w(1).
\]
Now assume that the induction hypothesis holds for $k$, we prove that it also holds for 
$k+1$. Note that
\begin{align}
\frac{1}{2}  
\sum_{\tilde{\bw}_{k+1}} \left | U(\tilde{\bw}_{k+1})
-  w(\tilde{\bw}_{k+1})   \right | = & \frac{1}{2}  
\sum_{\tilde{\bw}_{k}} \left[ \left | U(\tilde{\bw}_{k}) U(k+1)
-  w(\tilde{\bw}_{k}) w(k+1)  \right |   \right. \nonumber \\
+& \left. 
\left | U(\tilde{\bw}_{k}) (1-U(k+1))
-  w(\tilde{\bw}_{k}) (1-w(k+1))  \right |
\right] \nonumber \\
\stackrel{(a)}{\leq}  & \, \frac{1}{2}  
\sum_{\tilde{\bw}_{k}}
[  \left | U(\tilde{\bw}_{k}) U(k+1)
-  w(\tilde{\bw}_{k}) U(k+1)  \right | \nonumber \\
+& \,  \left | w(\tilde{\bw}_{k}) U(k+1)
-  w(\tilde{\bw}_{k}) w(k+1)  \right | \nonumber \\
+& \, 
\left | U(\tilde{\bw}_{k}) (1-U(k+1))
-  w(\tilde{\bw}_{k}) (1-U(k+1))  \right | \nonumber \\
+& \, 
\left | w(\tilde{\bw}_{k}) (1-U(k+1))
-  w(\tilde{\bw}_{k}) (1-w(k+1))  \right | 
] \nonumber \\
=& \, 
\frac{1}{2}  
\sum_{\tilde{\bw}_{k}}
[ U(k+1)  \left | U(\tilde{\bw}_{k}) 
-  w(\tilde{\bw}_{k})   \right | 
+ w(\tilde{\bw}_{k})  \left |  U(k+1)
-  w(k+1)  \right | \nonumber \\
+& \, 
(1-U(k+1)) \left | U(\tilde{\bw}_{k}) 
-  w(\tilde{\bw}_{k})  \right | 
+
w(\tilde{\bw}_{k}) \left | U(k+1)
-   w(k+1)  \right | 
] \nonumber \\
=& \, \frac{1}{2}  \sum_{\tilde{\bw}_{k}} \left | U(\tilde{\bw}_{k}) 
-  w(\tilde{\bw}_{k})  \right |  + \left [  U(k+1) - w(k+1)\right ] \nonumber \\
\stackrel{(b)}{\leq} & \,  \sum_{e =1}^k \left[U(e) - w(e) \right] + \left [  U(k+1) - w(k+1)\right ]  \nonumber \\
=& \,  \sum_{e =1}^{k+1} \left[U(e) - w(e) \right] ,
\end{align}
where (a) follows from the triangular inequality and (b) follows from the induction hypothesis. Hence,
we have proved Equation~\ref{eqn:induction} by induction hypothesis. As we have proved above, this is sufficient to prove Lemma~\ref{lemma:recursive_bound}.
\end{proof}

Finally, we prove the following lemma:
\begin{lemma}
\label{lemma:cumulative_bound}
For any two weight functions $w, U: \cE \rightarrow [0,1]$ s.t. $U(e) \geq w(e)$ for all $e \in \cE$, we have
\[
\textstyle f(\cS_t, U, v) -f(\cS_t, w, v) \leq 
\E \left[ 
\sum_{\tau=0}^{\tilde{\tau}-1} \sum_{e \in \cE(\cS^\tau, \cS^{0:\tau})} \left[U(e) - w(e) \right]  \middle | \cS_t
\right], 
\]
where $\tilde{\tau}$ is the stopping time when $\cS^\tau = \emptyset$ or $v \in \cS^\tau$, and the expectation is under the weight 
function $w$.
\end{lemma}
\begin{proof}
Recall that the diffusion process $\left( \cS^{\tau} \right )_{\tau=0}^{\tilde{\tau}}$ is a stochastic process. 
Note that by definition, if we treat the pair $(\cS^{\tau}, \cS^{0:\tau-1})$ as the \emph{state} of the diffusion process at diffusion step $\tau$, and assume
that $\bw(e) \sim \mathrm{Bern} \left( w(e) \right)$ are independently sampled for all $e \in \cE_{\cS_t, v}$, 
then the sequence $(\cS^{0}, \cS^{0:-1}), (\cS^{0}, \cS^{0:-1}), \ldots, (\cS^{\tilde{\tau}}, \cS^{0:\tilde{\tau}-1})$ follows a Markov chain, specifically,
\begin{itemize}
\item For any state $(\cS^{\tau}, \cS^{0:\tau-1})$ s.t. $v \notin \cS^{\tau}$ and $\cS^{\tau} \neq \emptyset$, its transition probabilities to the next state $(\cS^{\tau+1}, \cS^{0:\tau})$ depend on $w(e)$'s for $e \in \cE \left( \cS^{\tau} , \cS^{0:\tau} \right)$.
\item Any state $(\cS^{\tau}, \cS^{0:\tau-1})$ s.t. $v \in \cS^{\tau}$ or $\cS^{\tau} = \emptyset$ is a terminal state and the state transition terminates once visiting such a state. Recall that by definition of the stopping time $\tilde{\tau}$, the state transition terminates at $\tilde{\tau}$.
\end{itemize}

We define $h \left( \cS^\tau, \cS^{0:\tau-1}, U \right) - h \left( \cS^\tau, \cS^{0:\tau-1}, w \right) $ as the ``value" at state $(\cS^{\tau}, \cS^{0:\tau-1})$.
Also note that the states in this Markov chain is \emph{topologically sortable} in the sense that it will never revisit a state it visits before. Hence, we can compute 
$h \left( \cS^\tau, \cS^{0:\tau-1}, U \right) - h \left( \cS^\tau, \cS^{0:\tau-1}, w \right) $ via a backward induction from the terminal states, based on a valid topological order. 
Thus, from Lemma~\ref{lemma:recursive_bound}, we have
\begin{align}
f(\cS_t, U, v) -f(\cS_t, w, v) \stackrel{(a)}{=}& \, h(\cS^0, \emptyset, U)- h(\cS^0, \emptyset, w) \nonumber \\
\stackrel{(b)}{\leq}& \, \E \left[ 
\sum_{\tau=0}^{\tilde{\tau}-1} \sum_{e \in \cE(\cS^\tau, \cS^{0:\tau})} \left[U(e) - w(e) \right]  \middle | \cS^0
\right],
\end{align} 
where $(a)$ follows from the definition of $h$, and (b) follows from the backward induction. Since $\cS^0=\cS_t$ by definition, we have proved Lemma~\ref{lemma:cumulative_bound}.
\end{proof}

Finally, we prove Theorem~\ref{theorem:graph} based on Lemma~\ref{lemma:cumulative_bound}. Recall that the favorable event at round $t-1$
is defined as
\[
\xi_{t-1}= \left \{
|x_{e}\transpose (\bar{\theta}_{\tau-1} - \theta^*)| \leq c \sqrt{x_{e}\transpose \bM_{\tau-1}^{-1} x_{e}}, \, \forall e \in \cE, \, \forall \tau \leq t
\right \}.
\]
Also, based on Algorithm \ref{alg:imb},  we have $$0 \leq \bar{w}(e) \leq U_t(e) \leq 1,\forall e \in \cE. $$
Thus, from Lemma~\ref{lemma:cumulative_bound}, we have
\[
\textstyle f(\cS_t, U_t, v) -f(\cS_t, \bar{w}, v) \leq 
\E \left[ 
\sum_{\tau=0}^{\tilde{\tau}-1} \sum_{e \in \cE(\cS^\tau, \cS^{0:\tau})} \left[U_t(e) - \bar{w}(e) \right]  \middle | \cS_t, \cH_{t-1}
\right], 
\]
where the expectation is based on the weight function $\bar{w}$. 
Recall that $O_{t}(e)$ is the event that edge $e$ is observed at round $t$.
Recall that by definition, all edges in $\cE(\cS^\tau, \cS^{0:\tau})$ are observed at round $t$ (since they are going out from an influenced node in $\cS^\tau$, see Definition~\ref{def:O_t}) 
and belong to $\cE_{\cS_t, v}$, so we have
\begin{align}
\textstyle f(\cS_t, U_t, v) -f(\cS_t, \bar{w}, v) &\leq  \,  
\E \left[ 
\sum_{\tau=0}^{\tilde{\tau}-1} \sum_{e \in \cE(\cS^\tau, \cS^{0:\tau})} \left[U_t(e) - \bar{w}(e) \right]  \middle | \cS_t, \cH_{t-1}
\right] \nonumber \\
&\leq  
\E \left[ 
 \sum_{e \in \cE_{\cS_t, v}} \mathbf{1}\left(O_t(e)\right)\left[U_t(e) - \bar{w}(e) \right]  \middle | \cS_t, \cH_{t-1}
\right].
\end{align}
This completes the proof for Theorem~\ref{theorem:graph}.

\subsection{Proof of Lemma~\ref{lemma:worst}}
\label{sec:proof_lemma2}
\begin{proof}
To simplify the exposition, we define $z_{t,e}=\sqrt{x_{e}\transpose \bM_{t-1}^{-1} x_{e}}$ 
for all $t=1,2\ldots, n$ and all $e \in \cE$, and use $\cE_t^o$ denote the set of edges observed at round $t$.
Recall that
\begin{align}
\bM_t = \bM_{t-1} + \frac{1}{\sigma^2} \sum_{e \in \cE} x_{e}  x_{e}\transpose \mathbf{1} \left \{ O_t(e) \right \} =
\bM_{t-1} + \frac{1}{\sigma^2} \sum_{e \in \cE_t^o} x_{e}  x_{e}\transpose \label{eqn:M_update}.
\end{align}
Thus, for all $(t,e)$ such that $e \in \cE_t^o$ (i.e., edge $e$ is observed at round $t$), we have that
\begin{align}
\det \left[\bM_t \right] \geq & \det \left[ \bM_{t-1} + \frac{1}{\sigma^2} x_{e}  x_{e}\transpose \right]
= \det \left[ \bM_{t-1}^{\frac{1}{2}} \left ( \bI  + \frac{1}{\sigma^2} \bM_{t-1}^{- \frac{1}{2}} x_{e}  x_{e}\transpose  \bM_{t-1}^{- \frac{1}{2}} \right) \bM_{t-1}^{\frac{1}{2}}  \right] \nonumber \\
= & \det \left[ \bM_{t-1} \right] \det \left[\bI  + \frac{1}{\sigma^2} \bM_{t-1}^{- \frac{1}{2}} x_{e}  x_{e}\transpose  \bM_{t-1}^{- \frac{1}{2}} \right] \nonumber \\
= &\det \left[ \bM_{t-1} \right] \left( 1 + \frac{1}{\sigma^2}  x_{e}\transpose \bM_{t-1}^{- 1}  x_{e} \right)
= \det \left[ \bM_{t-1} \right]   \left( 1 + \frac{z_{t,e}^2}{\sigma^2}   \right). \nonumber
\end{align}
Thus, we have
\[
\left( \det \left[\bM_t \right] \right)^{|\cE_t^o|} \geq \left( \det \left[ \bM_{t-1} \right] \right)^{|\cE_t^o|}  \prod_{e \in \cE_t^o} \left( 1 + \frac{z_{t,e}^2}{\sigma^2}   \right).
\]
\begin{remark}
Notice that when the feature matrix $\bX=\bI$, $\bM_t$'s are always diagonal matrices, and we have
\[
 \det \left[\bM_t \right]  =  \det \left[ \bM_{t-1} \right]  \prod_{e \in \cE_t^o} \left( 1 + \frac{z_{t,e}^2}{\sigma^2}   \right),
\]
which will lead to a tighter bound in the tabular ($\bX=\bI$) case.
\end{remark}
Since 1) $\det \left[ \bM_t \right] \geq \det \left[ \bM_{t-1}  \right]$ from Equation~\ref{eqn:M_update} and 2) $|\cE_t^o| \leq \mre$,  where $\mre$ is defined in Equation~\ref{eqn:mre_def} and $|\cE_t^o| \leq \mre$ follows from its definition,  
we have
\[
\left( \det \left[\bM_t \right] \right)^{\mre} \geq \left( \det \left[ \bM_{t-1} \right] \right)^{\mre}  \prod_{e \in \cE_t^o} \left( 1 + \frac{z_{t,e}^2}{\sigma^2}   \right).
\]
Therefore, we have
\[
\left( \det \left[\bM_n \right] \right)^{\mre} \geq \left( \det \left[ \bM_{0} \right] \right)^{\mre} \prod_{t=1}^n \prod_{e \in \cE^o_t} \left( 1 + \frac{z_{t,e}^2}{\sigma^2}   \right) 
=\prod_{t=1}^n \prod_{e \in \cE^o_t} \left( 1 + \frac{z_{t,e}^2}{\sigma^2}   \right) ,
 \]
 since $\bM_0=\bI$. On the other hand, we have that
 \[
 \tr\left( \bM_n \right)=\tr \left(  \bI+ \frac{1}{\sigma^2} \sum_{t=1}^n \sum_{e \in \cE_t^o} x_{e}  x_{e}\transpose \right)
 = d +  \frac{1}{\sigma^2} \sum_{t=1}^n \sum_{e \in \cE_t^o} \| x_{e} \|_2^2 \leq d + \frac{n \mre}{\sigma^2}\CommaBin
 \]
 where the last inequality follows from the fact that $\| x_{e} \|_2 \leq 1$ and $ |\cE_t^o| \leq \mre$.
 From the trace-determinant inequality, we have $\frac{1}{d}  \tr\left( \bM_n \right) \geq \left[ \det(\bM_n)\right]^{\frac{1}{d}}$,
 thus we have
 \[
 \left[
 1+  \frac{n \mre}{d \sigma^2}
 \right]^{d \mre}
 \geq 
 \left [
 \frac{1}{d}  \tr\left( \bM_n \right) 
 \right ]^{d \mre}
 \geq
 \left[ \det( \bM_n)\right]^{\mre}
 \geq
  \prod_{t=1}^n \prod_{e \in \cE_t^o} \left( 1 + \frac{z_{t,e}^2}{\sigma^2}   \right).
 \]
Taking the logarithm on the both sides, we have
\begin{equation}
d \mre \log 
 \left[
 1+  \frac{n \mre}{d \sigma^2}
 \right]
 \geq
 \sum_{t=1}^n \sum_{e \in \cE_t^o} \log \left( 1 + \frac{z_{t,e}^2}{\sigma^2}   \right).
\end{equation}
Notice that $z_{t,e}^2=x_{e}\transpose \bM_{t-1}^{-1} x_{e} \leq x_{e}\transpose \bM_{0}^{-1} x_{e}
= \| x_{e} \|^2_2  \leq 1$, thus we have 
$z_{t,e}^2 \leq \frac{\log \left( 1 + \frac{z_{t,e}^2}{\sigma^2}   \right)}{\log \left( 1 + \frac{1}{\sigma^2}   \right)}\cdot$ \footnote{Notice that for any
$y \in [0,1]$, we have $y \leq \frac{\log \left( 1 + \frac{y}{\sigma^2}   \right)}{\log \left( 1 + \frac{1}{\sigma^2}   \right)} \stackrel{\Delta}{=} \kappa(y)$. 
To see it, notice that $\kappa(y)$ is a strictly concave function, and $\kappa (0)=0$ and $\kappa (1)=1$.}
Hence we have
\begin{align}
\sum_{t=1}^{n} \sum_{e \in \cE_t^o} z_{t,e}^2 
\leq \frac{1}{\log \left( 1 + \frac{1}{\sigma^2}   \right)}  \sum_{t=1}^n \sum_{e \in \cE_t^o} \log \left( 1 + \frac{z_{t,e}^2}{\sigma^2}   \right)
\leq \frac{d \mre \log 
 \left[
 1+  \frac{n \mre}{d \sigma^2}
 \right]}{\log \left( 1 + \frac{1}{\sigma^2}   \right)}.  \label{eqn:bound_z_1}
\end{align}
\begin{remark}
When the feature matrix $\bX=\bI$, we have $d = |\cE|$, 
\[
\det \left[\bM_n \right] 
= \prod_{t=1}^n \prod_{e \in \cE^o_t} \left( 1 + \frac{z_{t,e}^2}{\sigma^2}   \right),
\text{ and } \quad
|\cE| \log 
 \left[
 1+  \frac{n \mre}{ |\cE| \sigma^2}
 \right]
 \geq
 \sum_{t=1}^n \sum_{e \in \cE_t^o} \log \left( 1 + \frac{z_{t,e}^2}{\sigma^2}   \right).
\]
This implies that
\begin{align}
\sum_{t=1}^{n} \sum_{e \in \cE_t^o} z_{t,e}^2 
\leq \frac{|\cE| \log 
 \left[
 1+  \frac{n }{\sigma^2}
 \right]}{\log \left( 1 + \frac{1}{\sigma^2}   \right)}\CommaBin  \label{eqn:bound_z_2}
\end{align}
since $\mre \leq |\cE|$.
\end{remark}
Finally, from Cauchy-Schwarz inequality, we have that
\begin{align}
\sum_{t=1}^n \sum_{e \in \cE} 
 \mathbf{1} \{ O_{t}(e) \}  N_{\cS_t, e} \sqrt{x_{e}\transpose \bM_{t-1}^{-1} x_{e}} = & \,
 \sum_{t=1}^n \sum_{e \in \cE_t^o} 
N_{\cS_t, e} z_{t,e} \nonumber \\
\leq & \,
\sqrt{\left( \sum_{t=1}^n \sum_{e \in \cE_t^o} N^2_{\cS_t, e}  \right) \left( \sum_{t=1}^n \sum_{e \in \cE_t^o} z^2_{t,e}  \right)} \nonumber \\
=& \, \sqrt{\left( \sum_{t=1}^n \sum_{e \in \cE} \mathbf{1} \left \{ O_t(e) \right \} N^2_{\cS_t, e}  \right) \left( \sum_{t=1}^n \sum_{e \in \cE_t^o} z^2_{t,e}  \right)} .
\end{align}
Combining this inequality with the above bounds on $\sum_{t=1}^n \sum_{e \in \cE_t^o} z^2_{t,e}$ (see Equations \ref{eqn:bound_z_1} and \ref{eqn:bound_z_2}), we obtain the statement of the lemma. 
\end{proof}

\subsection{Proof of Lemma~\ref{lemma:concentration}}
\label{sec:proof_lemma3}
\begin{proof}
We use  $\cE_t^o$ denote the set of edges observed at round $t$.
The first observation is that we can order edges in $\cE_t^o$ based on breadth-first search (BFS) from the source nodes $\cS_t$, as described in Algorithm~\ref{alg:edge_sort}, where $\pi_t(\cS_t)$ is an arbitrary conditionally deterministic order of $\cS_t$.
We say a node $u \in \cV$ is a \emph{downstream neighbor} of node $v \in \cV$ if there is a directed edge $(v, u)$.
We also assume that there is a fixed order of downstream neighbors for any node $v \in \cV$.

\newcommand{\nqueue}{\mathrm{queueN}}
\newcommand{\equeue}{\mathrm{queueE}}
\newcommand{\ndict}{\mathrm{dictN}}
\newcommand{\edict}{\mathrm{dictE}}
\newcommand{\dequeue}{\mathrm{dequeue}()}
\newcommand{\enqueue}[1]{\mathrm{enqueue}(#1)}

\begin{algorithm}
\caption{Breadth-First Sort of Observed Edges}
\label{alg:edge_sort}
\begin{algorithmic}
\STATE \textbf{Input:} graph $\cG$, $\pi_t(\cS_t)$, and $\bw_t$
\vspace{0.15cm}
\STATE \textbf{Initialization:} node queue $\nqueue \leftarrow \pi_t (\cS_t)$, edge queue $\equeue \leftarrow \emptyset$, dictionary of influenced nodes $\ndict \leftarrow \cS_t$
\vspace{0.15cm}
\WHILE{$\nqueue$ is not empty}
\STATE node $v \leftarrow \nqueue.\dequeue$
\FOR{all downstream neighbor $u$ of $v$}
\STATE $\equeue.\enqueue{(v,u)}$
\IF{$\bw_t(v,u)==1$ and $u \notin \ndict$}
\STATE $\nqueue.\enqueue{u}$ and $\ndict \leftarrow \ndict \cup \{ u\}$
\ENDIF
\ENDFOR
\ENDWHILE
\vspace{0.15cm}
\STATE \textbf{Output:} edge queue $\equeue$
\end{algorithmic}
\end{algorithm}

Let $J_t=|\cE_t^o|$.
Based on Algorithm~\ref{alg:edge_sort}, 
we order the observed edges in $\cE_t^o$ as
$a^t_1, a^t_2, \ldots, a^t_{J_t}$.
We start by defining some useful notation. For any $t=1,2,\dots$, any $j=1,2, \dots, J_t$, we define
\[
\rnd{\eta}_{t,j}=\bw_t (a^t_j) - \bar{w} (a^t_j).
\]
One key observation is that $\rnd{\eta}_{t,j}$'s form a martingale difference sequence (MDS).\footnote{Notice that the notion of
``time" (or a round) is indexed by the pair $(t,j)$, and follows the lexicographical order. Based on Algorithm~\ref{alg:edge_sort},
at the beginning of round $(t,j)$, $a^t_j$ is conditionally deterministic and the conditional mean of 
$\bw_t (a^t_j)$ is $\bar{w} (a^t_j)$.
} Moreover,  $\rnd{\eta}_{t,j}$'s are bounded in $[-1,1]$
and hence they are conditionally sub-Gaussian with constant $R=1$. We further define that
\begin{align}
\bV_t = &  \sigma^2 \bM_t = \sigma^2 \bI + \sum_{\tau=1}^t \sum_{j=1}^{J_{\tau}} x_{a^{\tau}_j} x_{a^{\tau}_j}\transpose , \text{\ and}
\nonumber \\
Y_t = &  \sum_{\tau=1}^t \sum_{j=1}^{J_{\tau}} x_{a^{\tau}_j} \rnd{\eta}_{t,j}
= B_t - \sum_{\tau=1}^t \sum_{j=1}^{J_{\tau}} x_{a^{\tau}_j} \bar{w} (a^t_j)
=  B_t  - \left[  \sum_{\tau=1}^t \sum_{j=1}^{J_{\tau}} x_{a^{\tau}_j} x_{a^{\tau}_j}\transpose \right] \theta^*.
\nonumber
\end{align}
As we will see later, we define $\bV_t $ and $Y_t$ to use the  self-normalized bound developed in 
\cite{abbasi2011improved} (see Algorithm 1 of \cite{abbasi2011improved}). Notice that
\[
\bM_t \bar{\theta}_t = \frac{1}{\sigma^2} B_t =  \frac{1}{\sigma^2} Y_t + \frac{1}{\sigma^2}  \left[  \sum_{\tau=1}^t \sum_{j=1}^{J_{\tau}} x_{a^{\tau}_j} x_{a^{\tau}_j}\transpose \right] \theta^*
 = \frac{1}{\sigma^2} Y_t +  \left[ \bM_t -\bI \right] \theta^*,
\]
where the last equality is based on the definition of $\bM_t$. Hence we have
\[
\bar{\theta}_t  - \theta^* =  \bM_t^{-1} \left[ \frac{1}{\sigma^2} Y_t -\theta^* \right].
\]
Thus, for any $e \in \cE$, we have
\begin{align}
\left  |  \langle 
x_e, \bar{\theta}_t  - \theta^*
\rangle  \right  |
= & \left |
x_e\transpose
\bM_t^{-1} \left[ \frac{1}{\sigma^2} Y_t -\theta^* \right]
\right | 
\leq \| x_e \|_{\bM_t^{-1}} 
\|
\frac{1}{\sigma^2} Y_t -\theta^*
\|_{\bM_t^{-1}} \nonumber \\
 \leq &
\| x_e \|_{\bM_t^{-1}}  \left[
\|
\frac{1}{\sigma^2} Y_t \|_{\bM_t^{-1}} + \|\theta^*
\|_{\bM_t^{-1}}
\right]
 , \nonumber
\end{align}
where the first inequality follows from the Cauchy-Schwarz inequality and the second inequality follows from the triangle
inequality. Notice that $\|\theta^*
\|_{\bM_t^{-1}} \leq \|\theta^*
\|_{\bM_0^{-1}} = \|\theta^*
\|_2$, and 
$\|
\frac{1}{\sigma^2} Y_t \|_{\bM_t^{-1}}  = \frac{1}{\sigma} \| Y_t \|_{\bV_t^{-1}} $ (since $\bM_t^{-1}= \sigma^2 \bV_t^{-1}$), therefore we have
\begin{equation}
\left  |  \langle 
x_e, \bar{\theta}_t  - \theta^*
\rangle  \right  | \leq
\| x_e \|_{\bM_t^{-1}}  \left[
\frac{1}{\sigma} \| Y_t \|_{\bV_t^{-1}} + \|\theta^*
\|_2
\right ].
\label{eq:bound2}
\end{equation}
Notice that the above inequality always holds. We now provide a high-probability bound on $\| Y_t \|_{\bV_t^{-1}} $
based on self-normalized bound proved in \cite{abbasi2011improved}.
From Theorem 1 of \cite{abbasi2011improved}, we know that for any $\delta \in (0,1)$, with probability at least $1-\delta$, 
we have
\[
\| Y_t \|_{\bV_t^{-1}}  \leq
\sqrt{
2 \log \left(
\frac{\det(\bV_t)^{1/2} \det(\bV_0)^{-1/2}}
{\delta}
\right)
} \quad
\forall t=0,1,\dots \ .
\]
Notice that $ \det(\bV_0)= \det(\sigma^2 \bI)=\sigma^{2d}$. Moreover, from the trace-determinant inequality, we have
\[
\left[ \det (\bV_t) \right]^{1/d} \leq \frac{\tr \left ( \bV_t \right )}{d} = \sigma^2 + \frac{1}{d} \sum_{\tau=1}^t \sum_{j=1}^{J_{\tau}}
\| x_{a^\tau_j}\|_2^2 \leq \sigma^2 + \frac{t \mre }{d} \leq  \sigma^2 + \frac{n \mre}{d},
\]
where the second inequality follows from the assumption that $\| x_{\rnd{a}^t_k}\|_2 \leq 1$ and the fact $ J_t = |\cE^o_t| \leq \mre$, and the last inequality follows from $t \leq n$. Thus, with probability at least $1-\delta$, we have
\[
\| Y_t \|_{\bV_t^{-1}}  \leq \sqrt{d \log \left( 1 + \frac{n \mre}{d \sigma^2} \right)+ 2 \log \left (\frac{1}{\delta} \right) } \quad
\forall t=0,1,\dots, n-1.
\]
That is, with probability at least $1-\delta$, we have
\begin{equation}
\left  |  \langle 
x_e, \bar{\theta}_t  - \theta^*
\rangle  \right  | \leq
\| x_e \|_{\bM_t^{-1}}  \left[
\frac{1}{\sigma}  \sqrt{d \log \left( 1 + \frac{n \mre}{d \sigma^2} \right)+ 2 \log \left (\frac{1}{\delta} \right) } + \|\theta^*
\|_2
\right ] \nonumber
\end{equation}
for all $t=0,1,\dots, n-1$ and $\forall e \in E$. 


Recall that by the definition of event $\xi_{t-1}$, the above inequality implies that, for any $t=1,2,\ldots, n$, if
\[
c \geq \frac{1}{\sigma}  \sqrt{d \log \left( 1 + \frac{n \mre}{d \sigma^2} \right)+ 2 \log \left (\frac{1}{\delta} \right) } + \|\theta^*
\|_2,
\]
then $P (\xi_{t-1}) \geq 1-\delta$. That is, $P (\bar{\xi}_{t-1}) \leq \delta$.
\end{proof}

\end{document}